\documentclass{article} 
\usepackage[preprint]{colm2026/colm2026_conference}

\usepackage{microtype}
\usepackage{hyperref}
\usepackage{url}
\usepackage{booktabs}

\usepackage{lineno}

\definecolor{darkblue}{rgb}{0, 0, 0.5}
\hypersetup{colorlinks=true, citecolor=darkblue, linkcolor=darkblue, urlcolor=darkblue}

\usepackage{hyperref}
\usepackage{url}
\usepackage{algorithm}
\usepackage{algpseudocode}
\usepackage{amsmath}
\usepackage{amssymb}
\usepackage{bbm}
\usepackage{graphicx}
\usepackage{wrapfig}
\usepackage{microtype}
\usepackage{subcaption}
\usepackage{booktabs} 
\usepackage{multirow}
\usepackage{tabularx}
\usepackage{wrapfig}
\usepackage{xspace}
\usepackage{enumitem}
\usepackage{float}
\usepackage{marvosym}
\usepackage[nameinlink,capitalise]{cleveref}

\crefname{figure}{Fig.}{Figs.}
\Crefname{figure}{Fig.}{Figs.}

\crefname{table}{Tab.}{Tabs.}
\Crefname{table}{Tab.}{Tabs.}

\crefname{section}{Sec.}{Secs.}
\Crefname{section}{Sec.}{Secs.}

\crefname{appendix}{App.}{Apps.}
\Crefname{appendix}{App.}{Apps.}

\crefname{equation}{Eq.}{Eqs.}
\Crefname{equation}{Eq.}{Eqs.}

\newif\ifcomments
\commentstrue

\ifcomments
  \newcommand{\jl}[1]{\textcolor{red}{[JL: #1]}}
\else
  \newcommand{\jl}[1]{}
\fi

\ifcomments
  \newcommand{\zg}[1]{\textcolor{red}{[JL: #1]}}
\else
  \newcommand{\zg}[1]{}
\fi

\DeclareMathOperator*{\EE}{\mathbb{E}}

\newcommand{\method}{\textsc{Actor-Curator}\xspace}
\newcommand{\meth}{\textsc{AC}\xspace}

\newtheorem{lemma}{Lemma}
\newtheorem{theorem}{Theorem}
\newenvironment{proof}{\paragraph{Proof:}}{\hfill$\square$}

\title{\method: Co-adaptive Curriculum Learning via Policy-Improvement Bandits for Scalable RL Post-Training}

\author{
\parbox{0.9\linewidth}{
\textbf{Zhengyao Gu}$^{*\spadesuit}$\hspace{0.6em}
\textbf{Jonathan Light}$^{*\heartsuit\diamondsuit\ \text{\Letter}}$\hspace{0.6em}
Raul Astudillo$^{\clubsuit}$\hspace{0.6em}
Ziyu Ye$^{\spadesuit}$\hspace{0.6em}
Langzhou He$^{\spadesuit}$\hspace{0.6em}
Henry Peng Zou$^{\spadesuit}$\hspace{0.6em}
Wei Cheng$^{\blacklozenge}$\hspace{0.6em}
Santiago Paternain$^{\diamondsuit}$\hspace{0.6em}
Philip S.\ Yu$^{\spadesuit}$\hspace{0.6em}
Yisong Yue$^{\heartsuit}$
}\\[1.5em]
\parbox{0.95\linewidth}{
$^{\spadesuit}$University of Illinois Chicago \quad
$^{\heartsuit}$Caltech \quad
$^{\diamondsuit}$RPI \quad
$^{\clubsuit}$MBZUAI \quad
$^{\blacktriangle}$University of Chicago \quad
$^{\blacklozenge}$NEC Laboratories America
}\\[1.5em]
\parbox{0.95\linewidth}{
$^{*}$Equal contribution \\
\text{\Letter}\ Corresponding author: \texttt{jonathan.li.connect@gmail.com}
}
}

\begin{document}

\ifcolmsubmission
\linenumbers
\fi

\maketitle

\begin{abstract}
Post-training large foundation models with reinforcement learning typically involves selecting training problems from massive and heterogeneous datasets, where the choice of data has a critical impact on training stability, sample efficiency, and final performance.
In this work, we propose \method, a scalable and fully automated framework for reinforcement learning post-training of large language models (LLMs) that learns to adaptively curate training problems.
\method trains a neural \emph{curator} that dynamically selects problems from large problem banks by directly optimizing for expected policy performance improvement.
We formulate problem selection as a non-stationary stochastic bandit problem, derive a principled loss function based on online stochastic mirror descent, and establish regret guarantees under partial feedback.
Empirically, \method consistently outperforms uniform sampling and strong learning-based baselines across a wide range of challenging reasoning benchmarks, demonstrating improved training stability and efficiency.
Notably, it achieves relative gains of \textbf{28.6\% on AIME2024} and \textbf{30.5\% on ARC-1D} over the strongest baseline and up to \textbf{80\% speedup}.
These results suggest that \method provides a practical and principled approach to scalable, adaptive curriculum learning for LLM post-training.
\end{abstract}

\section{Introduction}

Reinforcement learning (RL) has become a central paradigm for post-training foundation models, enabling improvements in reasoning, alignment, and task-specific performance beyond supervised fine-tuning~\citep{shao2024deepseekmath}. In this setting, the choice, ordering, and frequency of training problems play a critical role in determining convergence speed, training stability, and final generalization performance, motivating the use of curriculum learning to adaptively select training data~\citep{bengio2009curriculum, tzannetos2023proximal, parashar2025curriculumreinforcementlearningeasy}. However, applying curriculum learning to modern foundation model post-training is challenging: post-training datasets are \textbf{large, diverse, and continuously evolving}, while actor updates induce \textbf{complex, non-stationary training dynamics}. Traditional curriculum learning approaches—based on manual difficulty annotations, hand-designed problem buckets, or tabular per-problem statistics~\citep{asada1996purposive, wu2017training, yengera2021curriculum}—do not scale to such settings, fail to generalize to unseen problems, and are brittle when problem utility changes as the policy improves. Moreover, effective curricula must balance \textbf{exploration} of under-sampled problems with \textbf{exploitation} of those that most improve the current policy, further complicating scalable curriculum design.

In this work, we propose \method (\meth), a scalable and fully automated problem curation framework for RL post-training that jointly trains an actor and a curator in an online, on-policy manner. At the core of \method is a learned \emph{curator} that adaptively selects training problems at each iteration and function-approximates over large, heterogeneous datasets. The curator is trained to directly maximize a \textbf{policy improvement objective}, assigning higher probability to problems expected to induce the greatest improvement in the actor’s performance. Unlike prior curricula that rely on heuristic signals such as absolute mean advantage or difficulty proxies~\citep{chen2025self, gao2025prompt, wang2025dump}, our objective is derived from expected policy improvement, providing a principled and actor-aware learning signal. As a result, \method naturally \emph{adapts to evolving actor training dynamics} and allows it to be seamlessly combined with a wide range of RL algorithms.

\begin{wrapfigure}{r}{0.5\linewidth}
    \centering
    \vspace{-0.25in}
    \includegraphics[width=\linewidth]{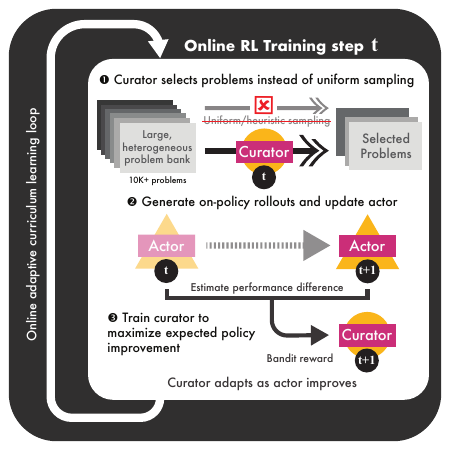}
    \vspace{-0.2in}
    \caption{\footnotesize\textbf{Co-adaptive online training loop of \method}. At each RL step, a learned curator adaptively selects problems from a large problem bank instead of uniform sampling. The actor is updated on these problems, after which a bandit-style reward based on post-update policy improvement trains the curator. As the actor improves, the curator adapts to prioritize problems that yield the greatest expected performance gains.}
    \label{fig:teaser}
    \vspace{-0.2in}
\end{wrapfigure}

To optimize the curator, we formalize problem selection as a \textbf{non-stationary stochastic bandit} problem with partial feedback. The curator is trained online and on-policy alongside the actor using \textbf{online stochastic mirror descent} (OSMD)~\cite{banditalgorithm}, which explicitly balances exploration and exploitation under non-stationarity. This differs from prior approaches that primarily rely on regression-style objectives and do not explicitly model the bandit structure or \textbf{partial observability} inherent in adaptive data selection~\citep{tzannetos2023proximal, gao2025prompt}.
To scale beyond tabular formulations, we derive a function-approximation variant of OSMD that trains the curator as a neural network, enabling generalization across problems and robustness to large, dynamic datasets. Finally, we introduce a PPO-style proximal clipping objective to stabilize curator optimization in practice.

Empirically, \method enables effective curriculum learning at scale \emph{without human annotations, difficulty labels, or manual dataset structuring}. Across diverse reasoning benchmarks—\textbf{Countdown, Zebra, MATH, AIME, and ARC-1D}—it consistently outperforms uniform sampling and strong baselines, achieving up to \textbf{30\%} higher peak performance on ARC-1D, \textbf{28\%} on AIME24, and up to \textbf{80\%} faster convergence to comparable performance. 


In summary, our main contributions are:
\begin{itemize}[leftmargin=*,itemsep=0pt,topsep=0pt,parsep=0.4pt,partopsep=0pt]
    \item \textbf{Automated problem curation framework for RL post-training.}
    We introduce \method, a scalable framework that learns a neural curator to adaptively select training problems in an online, on-policy manner, enabling curriculum learning over large, heterogeneous datasets without human annotations or manual structuring.

    \item \textbf{A policy-improvement–driven bandit formulation of data curation.}
    We cast problem selection as a non-stationary stochastic bandit problem and derive a principled learning signal grounded in policy improvement theory, optimized via an OSMD-based bandit objective with regret guarantees under partial feedback.
\end{itemize}

\vspace{-0.1in}
\section{Problem formulation}
\vspace{-0.1in}
\label{sec:problem}

We study problem curation for reinforcement learning (RL) post-training of large
language models (LLMs), where training is performed over large and heterogeneous
collections of problems.
Our goal is to design an adaptive data selection strategy that determines which training
problems an LLM should train on at each iteration in order to maximize overall post-training
performance.

\subsection{RL post-training setting}

Let $\mathcal{X} = \{\boldsymbol{x}^{(i)}\}_{i=1}^{|\mathcal{X}|}$ denote a large collection
of training problems, and let $p_{\mathcal{X}}$ be a fixed evaluation distribution over
$\mathcal{X}$.
Let $\pi$ denote a pretrained autoregressive language model, which induces a conditional distribution
$\boldsymbol{y} \sim \pi(\cdot \mid \boldsymbol{x})$ over solutions for each problem $\boldsymbol{x}$.
A reward model $R : \mathcal{Y} \times \mathcal{X} \rightarrow [0,1]$ assigns a scalar score to each solution.
The post-training objective is to maximize expected reward under the evaluation distribution:
\begin{equation}
\label{eq:grand_objective}
J(\pi)
\;\triangleq\;
\mathbb{E}_{\boldsymbol{x} \sim p_{\mathcal{X}},\, \boldsymbol{y} \sim \pi(\cdot \mid \boldsymbol{x})}
\bigl[ R(\boldsymbol{y} \mid \boldsymbol{x}) \bigr].
\end{equation}

In this work, the reinforcement learning algorithm, reward model, and rollout procedure
are fixed.
Our focus is on how training problems are selected across iterations.

\subsection{Training dynamics}

Training proceeds in iterations.
At iteration $t$, a \emph{curator} selects a subset of training problems
$\mathcal{X}^t \subset \mathcal{X}$.
Given this selection, the \emph{actor} $\pi^t$ is rolled out on each
$\boldsymbol{x}\in\mathcal{X}^t$ to produce solutions and corresponding rewards
\[
\mathcal{Y}_{\boldsymbol{x}}^t
\triangleq
\{\boldsymbol{y}^{(j)}\sim\pi^t(\cdot\mid\boldsymbol{x})\}_{j=1}^{|\mathcal{Y}_{\boldsymbol{x}}|},
\qquad
\mathcal{R}_{\boldsymbol{x}}^t
\triangleq
\{R(\boldsymbol{y}^{(j)}\mid\boldsymbol{x})\}_{j=1}^{|\mathcal{Y}_{\boldsymbol{x}}|}.
\]

These trajectories form the dataset
\begin{equation}
\label{eq:dataset_definition}
\mathcal{D}^t
\triangleq
\bigl\{
(\boldsymbol{x}, \mathcal{Y}_{\boldsymbol{x}}^t, \mathcal{R}_{\boldsymbol{x}}^t)
\mid \boldsymbol{x}\in\mathcal{X}^t
\bigr\},
\qquad
\pi^{t+1} \leftarrow \mathcal{A}(\pi^t, \mathcal{D}^t).
\end{equation}

Here $\mathcal{A}$ may correspond to any standard post-training algorithm
(e.g., GRPO~\citep{shao2024deepseekmath} or GSPO~\citep{ahmadian2024back}).
We emphasize that the curator influences learning only indirectly through data selection,
while the actor is solely responsible for policy optimization.

\subsection{Curriculum learning as problem selection}

The central problem addressed in this work is how to choose the training subsets
$\mathcal{X}^t$ across iterations.
Different choices of $\mathcal{X}^t$ induce different actor updates and therefore
different trajectories of policy improvement.
We formalize curriculum learning as a sequential decision-making problem.
At each iteration $t$, a \emph{curator} selects a subset of training problems
$\mathcal{X}^t \subset \mathcal{X}$.
This selection induces a performance improvement
$J(\pi^{t+1}) - J(\pi^t)$.
The curator’s objective is to maximize cumulative performance gains over training:
\vspace{-0.1in}
\begin{equation}
\label{eq:curator_objective}
\max_{\{\mathcal{X}^t\}_{t=1}^T}
\sum_{t=1}^T
\bigl(
J(\pi^{t+1}) - J(\pi^t)
\bigr).
\end{equation}
\vspace{-0.2in}
\subsection{Challenges of adaptive problem selection}

Effective curation is challenging for several reasons:

\begin{itemize}[leftmargin=*,itemsep=0pt,topsep=0pt,parsep=0.4pt,partopsep=0pt]
    \item \textbf{Large action space.} The problem set $\mathcal{X}$ is large and might change across time, making it
    infeasible to manually define curricula or track per-problem statistics.

    \item \textbf{Partial feedback.} At each iteration, feedback is observed only for the
    problems selected for training; the utility of unselected problems remains unknown.

    \item \textbf{Non-stationarity.} The usefulness of a problem depends on the current
    actor $\pi^t$ and changes as the actor improves. It is also highly dependent on the actor update method. 

    \item \textbf{Exploration--exploitation trade-off.} The curator must balance exploring under-sampled problems whose utility is uncertain with exploiting problems that are known to drive policy improvement.
\end{itemize}

These challenges motivate a curriculum learning approach that operates at scale, learns online from partial feedback, and explicitly accounts for the non-stationary relationship between training problems and policy improvement.

\vspace{-0.1in}
\section{Method}
\vspace{-0.1in}
\label{sec:methodology}

\begin{figure*}[htbp]
\vspace{-0.1in}
    \centering
    \includegraphics[width=\textwidth]{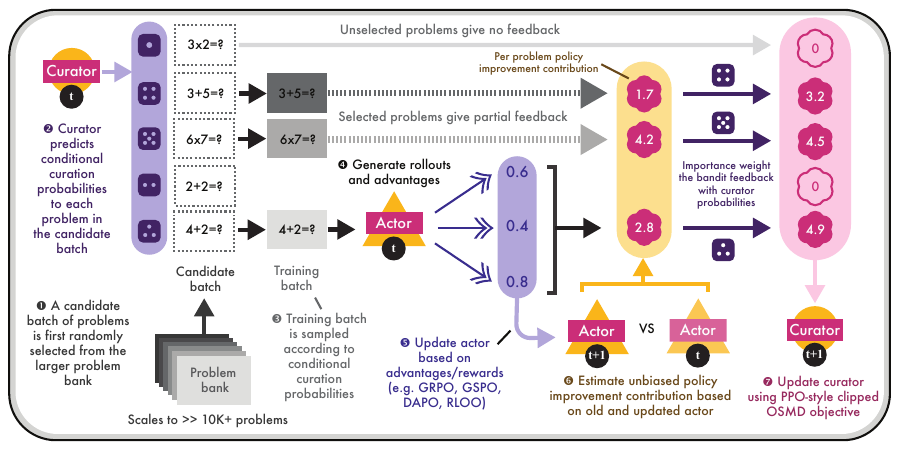}
    \vspace{-0.2in}
    \caption{
    \footnotesize
    \textbf{Single training iteration of \method}.
    At each iteration, a candidate set of problems is sampled from a fixed proposal distribution.
    The curator reweights this candidate set to select training problems for the actor.
    After the actor update, per-problem policy improvement is estimated using pre- and post-update policies.
    The curator observes bandit feedback only on selected problems and is updated using a PPO-style approximation of online stochastic mirror descent.
    }
    \label{fig:main}
    \vspace{-0.1in}
\end{figure*}

We now present \method, a curriculum learning framework that trains a learned \emph{curator}
to adaptively select training problems for RL post-training of large language models.
The key idea is to treat problem selection as a non-stationary bandit problem and to train
the curator to directly maximize \emph{policy improvement}—the expected performance gain
induced by each actor update—using online stochastic mirror descent (OSMD) under partial feedback.

\begin{algorithm}[h]
\caption{Actor--Curator: Self-driven curriculum learning}
\label{alg:curriculum}
\begin{algorithmic}[1]
\State \textbf{Input:} dataset $\mathcal{X}$, pretrained LLM (actor) $\pi^0$, reward model $R$, curator model $C_{\phi^0}$, proposal distribution $\tilde{q}$
\State \textbf{Hyperparameters:} number of training steps $T$, rollouts per problem $|\mathcal{Y}_{\boldsymbol{x}}|$, candidate batch size $|\tilde{\mathcal{X}}^t|$, training batch size $|\mathcal{X}^t|$

\For{$t = 0$ to $T-1$} \Comment{Training steps}
    \State $\mathcal{D}^t \gets \emptyset$
    \State \textbf{Proposal step:} sample a candidate batch $\tilde{\mathcal{X}}^t \subset \mathcal{X}$ according to $\tilde{q}$
    \State \textbf{Selection step:} sample a training batch $\mathcal{X}^t \subset \tilde{\mathcal{X}}^t$ using the curator $C_{\phi^t}$

    \For{each problem $\boldsymbol{x} \in \mathcal{X}^t$}
        \State Roll out solutions $\mathcal{Y}^t_{\boldsymbol{x}} = \{\boldsymbol{y}^{(j)}\sim \pi^{t}(\cdot \mid \boldsymbol{x}) \}_{j=1}^{|\mathcal{Y}_{\boldsymbol{x}}|}$
        \State Compute rewards $\mathcal{R}^t_{\boldsymbol{x}} = \{R(\boldsymbol{y}^{(j)} \mid \boldsymbol{x} )\}_{j=1}^{|\mathcal{Y}_{\boldsymbol{x}}|}$
        \State Add $(\boldsymbol{x}, \mathcal{Y}^t_{\boldsymbol{x}}, \mathcal{R}^t_{\boldsymbol{x}})$ to $\mathcal{D}^t$
    \EndFor

    \State \textbf{Actor update:} $\pi^{t+1} \leftarrow \mathcal{A}(\pi^{t}, \mathcal{D}^t)$
    \State \textbf{Curator utilities:} compute $\hat{\boldsymbol{U}}^t$ using \eqref{eq:x_contribution_estimated_two_stage} (and $\pi^{t+1}$)
    \State \textbf{Curator update:} $\phi^{t+1} \leftarrow \arg\min_{\phi}\ \mathcal{L}_{\text{cur}}(\phi)$ using \eqref{eq:clipped_curator}
\EndFor
\end{algorithmic}
\end{algorithm}

\subsection{Overview of the training loop}

Training proceeds in iterations.
At iteration $t$, the following steps are performed:

\begin{enumerate}[leftmargin=*,itemsep=0pt,topsep=0pt,parsep=0.4pt,partopsep=0pt]
    \item We sample a training subset $\mathcal{X}^t \subset \mathcal{X}$ based on probabilities produced by the curator.
    \item The actor $\pi^t$ is rolled out on $\mathcal{X}^t$ to collect trajectories and updated to $\pi^{t+1}$ using any RL update.
    \item Compute a per-problem policy improvement estimate  for problems in $\mathcal{X}^t$ using $\pi^t$ and $\pi^{t+1}$
    \item The curator is updated online using bandit feedback derived from these improvement estimates.
\end{enumerate}

The curator is trained jointly with the actor in an on-policy manner, allowing the curriculum
to adapt dynamically as the actor improves.
Figure~\ref{fig:main} illustrates this process. Pseudocode is provided in \cref{alg:curriculum}.

\subsection{Policy improvement as the curator learning signal}

The curator’s objective is to select problems that maximize improvement in the actor’s
performance under the fixed evaluation distribution $p_{\mathcal{X}}$.

\paragraph{Performance improvement identity.}
Define the performance improvement at iteration $t$ as
\[
u^t \triangleq J(\pi^{t+1}) - J(\pi^t).
\]
In the single-turn setting, the performance difference identity~\citep{kakade2002approximately}
gives
\begin{equation}
\label{eq:policy_improvement_identity}
u^t
=
\mathbb{E}_{\boldsymbol{x}\sim p_{\mathcal{X}}}
\mathbb{E}_{\boldsymbol{y}\sim \pi^{t+1}(\cdot\mid\boldsymbol{x})}
\left[
A_{\pi^t}(\boldsymbol{y}\mid\boldsymbol{x})
\right]
, \quad
A_{\pi^t}(\boldsymbol{y}\mid\boldsymbol{x})
\triangleq
R(\boldsymbol{y}\mid\boldsymbol{x})
-
\mathbb{E}_{\boldsymbol{y}'\sim\pi^t(\cdot\mid\boldsymbol{x})}
\left[
R(\boldsymbol{y}'\mid\boldsymbol{x})
\right].
\end{equation}
Applying importance sampling yields
\begin{equation}
\label{eq:policy_improvement_is}
u^t
=
\mathbb{E}_{\boldsymbol{x}\sim p_{\mathcal{X}}}
\mathbb{E}_{\boldsymbol{y}\sim \pi^{t}(\cdot\mid\boldsymbol{x})}
\left[
\frac{\pi^{t+1}(\boldsymbol{y}\mid\boldsymbol{x})}
{\pi^{t}(\boldsymbol{y}\mid\boldsymbol{x})}
A_{\pi^t}(\boldsymbol{y}\mid\boldsymbol{x})
\right].
\end{equation}
\paragraph{Per-problem utility.}
\Cref{eq:policy_improvement_is} decomposes additively across problems.
For each $\boldsymbol{x}\in\mathcal{X}$, define the per-problem utility
\begin{equation}
\label{eq:per_problem_utility}
u_{\boldsymbol{x}}^t
\triangleq
p_{\mathcal{X}}(\boldsymbol{x})
\mathbb{E}_{\boldsymbol{y}\sim \pi^{t}(\cdot\mid\boldsymbol{x})}
\left[
\frac{\pi^{t+1}(\boldsymbol{y}\mid\boldsymbol{x})}
{\pi^{t}(\boldsymbol{y}\mid\boldsymbol{x})}
A_{\pi^t}(\boldsymbol{y}\mid\boldsymbol{x})
\right].
\end{equation}
By construction, $u^t = \sum_{\boldsymbol{x}} u_{\boldsymbol{x}}^t$.
Although the actor update couples all selected problems, $u_{\boldsymbol{x}}^t$ provides a
principled first-order credit assignment signal under small policy updates, grounded in
policy improvement theory. We provide further explanation in \cref{sec:utility_contribution}.

\subsection{Tabular OSMD formulation}
\vspace{-0.1in}

We cast curriculum learning as a non-stationary stochastic bandit problem, where each
problem $\boldsymbol{x}$ corresponds to an arm with time-varying utility $u_{\boldsymbol{x}}^t$.
The curator is optimized using \textbf{online stochastic mirror descent} (OSMD) under bandit feedback ~\citep{banditalgorithm}. We start with the tabular formulation first, where the curator maintains a probability mass function $\boldsymbol{p}^t \in \Delta_\alpha(\mathcal{X)}$ over a finite set of problems, where $\boldsymbol{p}^t$ is clipped to  the sampling distribution $p^t(\boldsymbol{x}\mid\tilde{\mathcal{X}}^t) \ge \alpha > 0$.
In \cref{sec:func_approx} we show how to represent $p^t$ using a learned model.

\paragraph{Utility estimation and OSMD bandit feedback.}
For each $\boldsymbol{x}\in\mathcal{X}^t$, let $\mathcal{Y}_{\boldsymbol{x}}^t$ denote rollouts
from $\pi^t(\cdot\mid\boldsymbol{x})$.
We estimate \cref{eq:per_problem_utility} via
\begin{equation}
\label{eq:x_contribution_estimated}
    \hat{U}^t_{\boldsymbol{x}} \triangleq  p_\mathcal{X}(\boldsymbol{x})
    \frac{\mathbb{I}\{\boldsymbol{x} \in \mathcal{X}^t\}}{p^t(\boldsymbol{x})} 
    \hat{A}^t(\cdot \mid \boldsymbol{x})
    , \qquad
    \hat{A}^t(\cdot \mid \boldsymbol{x}) \triangleq \frac{1}{|\mathcal{Y}^t_{\boldsymbol{x}}|} \sum_{\boldsymbol{y} \in \mathcal{Y}^t_{\boldsymbol{x}}} \frac{\pi^{t+1}(\boldsymbol{y}\mid \boldsymbol{x})}{\pi^{t}(\boldsymbol{y}\mid \boldsymbol{x})} A(\boldsymbol{y}\mid \boldsymbol{x})
\end{equation}
where $\hat{A}^t$ is the importance normalized estimated average advantage. This estimate is agnostic to the specifics of the actor update method as long as an updated actor is produced. 

\begin{theorem}[Unbiasedness]
\label{thm:one_stage_unbiased}
$\mathbb{E}[\hat{U}_{\boldsymbol{x}}^t]=u_{\boldsymbol{x}}^t$.
\end{theorem}
We prove this in \cref{sec:performance_bound}.
For decoder language models, $\pi^{t+1}(\boldsymbol{y}\mid\boldsymbol{x})$ is obtained via a single forward pass of the updated model on the previous solution.

\paragraph{OSMD update}

Given bandit feedback, the curator is updated with a
negative-entropy regularizer:
\begin{equation}
\label{eq:osmd_update}
p^{t+1}
\leftarrow
\arg\min_{\boldsymbol{p}\in\Delta_\alpha(\mathcal{X})}
\left\{
-\eta \langle \boldsymbol{p}, \hat{\boldsymbol{U}}^t\rangle
+
\mathrm{KL}(\boldsymbol{p}\|\boldsymbol{p}^t)
\right\}.
\end{equation}
This yields the exponentiated-gradient update
$
p^{t+1}(\boldsymbol{x})
\propto
p^t(\boldsymbol{x})\exp(\eta \hat{U}_{\boldsymbol{x}}^t).
$

\subsection{Function approximation for curator training}
\label{sec:func_approx}

Although the OSMD update in \cref{eq:osmd_update} is defined over a distribution on the
entire problem set $\mathcal{X}$, explicitly maintaining and updating tabular probabilities
is infeasible when $\mathcal{X}$ is large.
We therefore parameterize the curator using a neural network that assigns a positive score
to each problem and implicitly defines a probability distribution.
The curator and induced distribution are defined as
\[
C_\phi : \boldsymbol{x} \mapsto w_\phi(\boldsymbol{x}), \qquad
p_\phi(\boldsymbol{x})
\triangleq
\frac{w_\phi(\boldsymbol{x})}{\sum_{\boldsymbol{x}'\in\mathcal{X}} w_\phi(\boldsymbol{x}')},
\quad
w_\phi(\boldsymbol{x}) > 0.
\]

\vspace{-0.2in}
\paragraph{OSMD surrogate objective.}
To implement the OSMD update \cref{eq:osmd_update} with function approximation, we optimize
the following surrogate objective:
\begin{equation}
\label{eq:curator_surrogate}
\mathcal{L}_{\mathrm{cur}}(\phi)
=
\mathrm{KL}\!\left(
p_\phi
\;\|\;
p^t
\right)
-
\eta
\langle
p_\phi,
\hat{\boldsymbol{U}}^t
\rangle ,
\end{equation}
where $p^t$ denotes the curator distribution from iteration $t$.

\subsection{Two-stage sampling}
\label{sec:two_stage_sampling}

Sampling directly from a curator distribution over the full problem set $\mathcal{X}$ is
computationally infeasible at scale.
We therefore adopt a two-stage sampling scheme that separates \emph{coverage} from
\emph{adaptive curation} while preserving unbiased utility estimation.

At iteration $t$, we first sample a candidate set
$\tilde{\mathcal{X}}^t \subset \mathcal{X}$ from a fixed proposal distribution $\tilde{q}$.
Let
\[
q(\boldsymbol{x})
\triangleq
\Pr_{\tilde{\mathcal{X}}\sim\tilde{q}}\!\left(\boldsymbol{x}\in\tilde{\mathcal{X}}\right),
\quad
p^t(\boldsymbol{x}\mid\tilde{\mathcal{X}}^t)
\triangleq
\frac{w^t(\boldsymbol{x})}
{\sum_{\boldsymbol{x}'\in\tilde{\mathcal{X}}^t} w^t(\boldsymbol{x}')}.
\]
Here $q(\boldsymbol{x})$ denotes the induced marginal inclusion probability, and we assume
$q(\boldsymbol{x}) \ge q_{\min} > 0$ for all $\boldsymbol{x}\in\mathcal{X}$.
Conditioned on $\tilde{\mathcal{X}}^t$, the curator samples a training set
$\mathcal{X}^t \subset \tilde{\mathcal{X}}^t$ according to the restricted distribution
$p^t(\boldsymbol{x}\mid\tilde{\mathcal{X}}^t)$, where $w^t(\boldsymbol{x})>0$
is the curator score at iteration $t$.
This allows the curator to prioritize problems while operating only on a small candidate batch.

\paragraph{Utility estimation.}
The unbiased two-stage estimator corresponding to \cref{eq:x_contribution_estimated} is
\begin{equation}
\label{eq:x_contribution_estimated_two_stage}
\hat{U}^{t}_{\mathrm{two}, \boldsymbol{x}}
\;\triangleq\;
p_\mathcal{X}(\boldsymbol{x})\,
\frac{\mathbb{I}\{\boldsymbol{x} \in \mathcal{X}^t\}}
{q(\boldsymbol{x})\,p^t(\boldsymbol{x} \mid \tilde{\mathcal{X}}^t)}
\;\hat{A}^t(\cdot \mid \boldsymbol{x}),
\end{equation}
which corrects for both proposal and curation sampling probabilities.
This estimator satisfies $\mathbb{E}[\hat{U}^{t}_{\mathrm{two}, \boldsymbol{x}}]=u^t_{\boldsymbol{x}}$
(see \cref{sec:performance_bound}).
Substituting \cref{eq:x_contribution_estimated_two_stage} into the OSMD update \cref{eq:osmd_update} yields the
surrogate objective
\begin{equation}
\label{eq:curator_surrogate_two}
\begin{aligned}
\mathcal{L}_{\mathrm{two}}(\phi)
&\triangleq
\mathrm{KL}\!\left(
p_\phi \;\|\; p^t
\right)
-
\eta
\sum_{\boldsymbol{x}\in\mathcal{X}^t}
\frac{p_\phi(\boldsymbol{x}\mid\tilde{\mathcal{X}}^t)}
{p^t(\boldsymbol{x}\mid\tilde{\mathcal{X}}^t)}
\;
\frac{p_\mathcal{X}(\boldsymbol{x})\,\hat{A}^t(\cdot\mid\boldsymbol{x})}
{q(\boldsymbol{x})},
\end{aligned}
\end{equation}
where $p_\phi(\cdot\mid\tilde{\mathcal{X}}^t)$ is the conditional curator distribution.

\paragraph{Regret guarantee.}
We now state a regret bound for curator optimization under two-stage sampling.
The bound characterizes the curator's ability to track the best sequence of
problem-selection distributions in hindsight despite non-stationary utilities. A proof is provided in \cref{sec:idealized-omd}.

\begin{theorem}
\label{thm:regret}
Assume the curator is updated using OSMD with a negative-entropy regularizer and
receives bandit feedback
$\hat{U}^{t}_{\mathrm{two}, \boldsymbol{x}}$ satisfying
$\mathbb{E}[\hat{U}^{t}_{\mathrm{two}, \boldsymbol{x}}] = u^t_{\boldsymbol{x}}$.
Then the cumulative dynamic regret satisfies
\[
\mathrm{Reg}_T
\;\le\;
O\!\left(T^{2/3} V_T^{1/3}\right), \quad V_T \triangleq \sum_{t=2}^T \max_{\boldsymbol{x} \in \mathcal{X}} |u_{\boldsymbol{x}}^t - u_{\boldsymbol{x}}^{t-1}|
\]
where $\mathrm{Reg}_T$ is the regret against the best available arm, ignoring uniform exploration which we define in \cref{sec:regret-proof}, and $V_T$ is a measure of how rapid the utility of a problem changes over time $t$.
\end{theorem}

While we focus on two-stage sampling for efficiency,
\method is compatible with other approximate sampling schemes
(e.g., Metropolis--Hastings), provided marginal inclusion
probabilities are roughly proportional to curator-assigned weights.
\vspace{-3mm}

\subsection{Proximal curator optimization}
\label{sec:proximal_curation}
Directly optimizing the KL-regularized objective in~\eqref{eq:curator_surrogate} can be
unstable with neural network parameterization ~\citep{schulman2015trust}.
Following proximal policy optimization (PPO)~\citep{schulman2017proximal}, we adopt a
clipped surrogate objective.
Define the importance ratio and sub-objective as
\[
\rho_\phi(\boldsymbol{x})
\triangleq
\frac{p_\phi(\boldsymbol{x}\mid\tilde{\mathcal{X}}^t)}
{p^t(\boldsymbol{x}\mid\tilde{\mathcal{X}}^t)}, 
\quad
g^t(\boldsymbol{x})
\triangleq
\frac{p_\mathcal{X}(\boldsymbol{x})\hat{A}^t(\cdot \mid \boldsymbol{x})}{q(\boldsymbol{x})} 
\]
Starting from ~\eqref{eq:curator_surrogate_two}, we replace the linear probability ratio term with a clipped surrogate
\begin{equation}
\label{eq:clipped_curator}
\begin{aligned}
    \mathcal{L}_{\mathrm{cur}}^{\mathrm{PCO}}(\phi)
\triangleq
-\eta
\sum_{\boldsymbol{x}\in\tilde{\mathcal{X}}^t}
\min\!\Big(&
\rho_\phi(\boldsymbol{x})\,g^t(\boldsymbol{x}),\;
\mathrm{clip}(\rho_\phi(\boldsymbol{x}),\rho_{\min}, \rho_{\max})\,
g^t(\boldsymbol{x})
\Big),
\end{aligned}
\end{equation}
where $\rho_{\min}, \rho_{\max}$ are clipping parameters.
This objective constrains successive curator updates while preserving the behavior of online mirror descent in practice.

\section{Experimental results}

\paragraph{Benchmarks.}
We evaluate \method on five reasoning and mathematics benchmarks.
\textbf{Countdown} is an arithmetic puzzle requiring the composition of integers and operations to reach a target value~\citep{stojanovski2025reasoning}.
\textbf{Zebra} is a symbolic logic puzzle that requires finding assignments satisfying a set of constraints~\citep{stojanovski2025reasoning}.
\textbf{ARC-1D} is the one-dimensional variant of the Abstraction and Reasoning Corpus, designed to test inductive generalization~\citep{chollet2019measure,xu2023llms}.
\textbf{MATH500} consists of competition-level mathematics problems~\citep{hendrycks2021measuring}.
\textbf{AIME24} contains problems from the 2024 American Invitational Mathematics Examination.
We additionally consider hard subsets (\textbf{countdown-hard}, \textbf{zebra-hard}, \textbf{arc-hard}) for validation.
We train on \textbf{30K} problems for Countdown, Zebra, and ARC-1D, and \textbf{12K} MATH problems for MATH500 and AIME24.
\vspace{-3mm}

\paragraph{Experimental setup.}
We implement our post-training pipeline using VERL ~\citep{sheng2024hybridflow}.
The curator is initialized from a pretrained Qwen3-0.6B model~\citep{yang2025qwen3}.
Unless otherwise specified, the actor is trained using GSPO, a stabilized variant of GRPO~\citep{ahmadian2024back}, on Qwen2.5-3B.
Additional details are provided in \cref{sec:experimental_setup}, including hyper-parameters.
\begin{table*}[t]
  \caption{\textbf{Peak validation performance} on problems within 100 training steps for different methods with qwen2.5-3b. \method outperforms both other learning based methods (PCL) and methods that rely on human heuristics (SEC). We see similar results with llama3.2-3b-instruct, as shown in \cref{tab:performance_full}.}
  \label{tab:performance}
  \vspace{-0.1in}
  \begin{center}
    \begin{small}
      \begin{sc}
        \setlength{\tabcolsep}{4pt}
        \renewcommand{\arraystretch}{1.05}
        \begin{tabular}{l l ccccccc}
          \toprule
          \textbf{$|\mathcal{X}|$}
          & \textbf{Benchmark}
          & \multicolumn{5}{c}{\textbf{Method}}
          & \multicolumn{2}{c}{\textbf{Improvement}} \\
          \cmidrule(lr){3-7}\cmidrule(lr){8-9}
          &
          &
          $\boldsymbol{\pi_{\mathrm{ref}}}$
          & \textbf{Uniform}
          & \textbf{SEC}
          & \textbf{PCL}
          & \textbf{\meth (Ours)}
          & \textbf{$+\Delta$}
          & \textbf{$+\Delta\%$} \\
          \midrule

          \multirow{2}{*}{30,000}
          & Countdown
          & 0.00 & 44.74 & 58.87 & 57.24 & \textbf{62.12} & +3.25 & +5.52 \\

          & Countdown-hard
          & 0.00 & 41.00 & 51.50 & 48.00 & \textbf{58.00} & +6.50 & +12.62 \\
          \midrule

          \multirow{2}{*}{30,000}
          & Zebra
          & 0.00 & 35.12 & 36.00 & 34.12 & \textbf{37.62} & +1.62 & +4.50 \\

          & Zebra-hard
          & 0.00 & 30.50 & 27.50 & 26.00 & \textbf{34.50} & +4.00 & +13.11 \\
          \midrule

          \multirow{2}{*}{30,000}
          & ARC-1D
          & 0.00 & 26.74 & 27.87 & 26.37 & \textbf{36.37} & +8.50 & +30.51 \\

          & ARC-hard
          & 0.00 & 19.50 & 18.50 & 18.50 & \textbf{31.00} & +11.50 & +58.97 \\
          \midrule

          \multirow{2}{*}{12,000}
          & MATH500
          & 61.80 & \textbf{83.00} & 81.00 & 79.79 & 81.00 & -2.00 & -2.41 \\

          & AIME24
          & 3.33 & 23.33 & 20.00 & 23.33 & \textbf{30.00} & +6.67 & +28.57 \\

          \bottomrule
        \end{tabular}
      \end{sc}
    \end{small}
  \end{center}
  \vskip -0.1in
\end{table*}

\subsection{Main results}
\label{sec:main_results}
\paragraph{Baselines.}
We compare against state-of-the-art curriculum learning methods for LLM post-training, using the same backbone model and actor update.
$\pi_{\text{ref}}$ denotes the pretrained model without post-training.
\textbf{Uniform sampling} draws training problems uniformly at random.
\textbf{SEC} partitions problems into manually defined buckets and updates bucket probabilities based on the sum of absolute advantages~\citep{chen2025self}.
\textbf{PCL} trains a value model to estimate success probabilities and prioritizes problems with predicted success near $50\%$~\citep{gao2025prompt}.
PCL is competitive with recent curriculum-based approaches~\citep{yue2025vapo,zhang2025speed,zheng2025act}.

\paragraph{Performance.}
We evaluate performance on held-out test sets, recording metrics every 10 training steps.
Following prior work~\citep{gao2025prompt}, we report the peak performance achieved within the first 100 steps.
Results are summarized in \cref{tab:performance}.
Across both backbone models and most benchmarks, \method consistently outperforms all baselines, with additional results in \cref{sec:main_results_plus}.
Notably, \method achieves substantially larger gains on harder benchmarks such as \textbf{arc-hard} and \textbf{AIME24}, indicating that adaptive curation is particularly beneficial in challenging regimes.
\vspace{-3mm}

\paragraph{Efficiency.}
As shown in \cref{fig:efficiency}, \method reaches comparable or higher performance using significantly fewer training steps than uniform sampling, demonstrating improved sample efficiency.

\paragraph{Training dynamics.}
Across datasets, \method exhibits more stable optimization and often continues to improve performance after baselines plateau, as shown in \cref{fig:training_curves,fig:testcurves-qwen25-3b,fig:testcurves-llama32-3b}.

\begin{figure*}[htbp]
    \centering
    \begin{subfigure}[b]{0.32\textwidth}
        \centering
        \includegraphics[width=\textwidth]{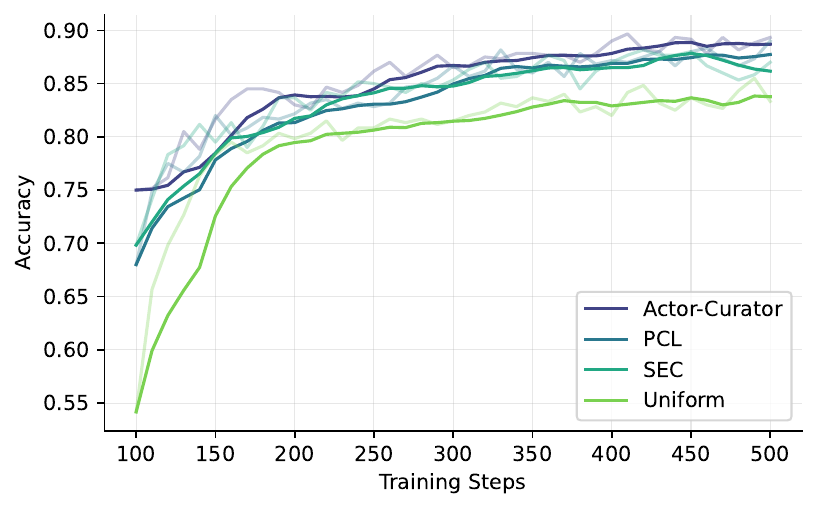}
        \caption{Countdown training curves}
        \label{fig:countdown_training_dynamics}
    \end{subfigure}
    \hfill
    \begin{subfigure}[b]{0.32\textwidth}
        \centering
        \includegraphics[width=\textwidth]{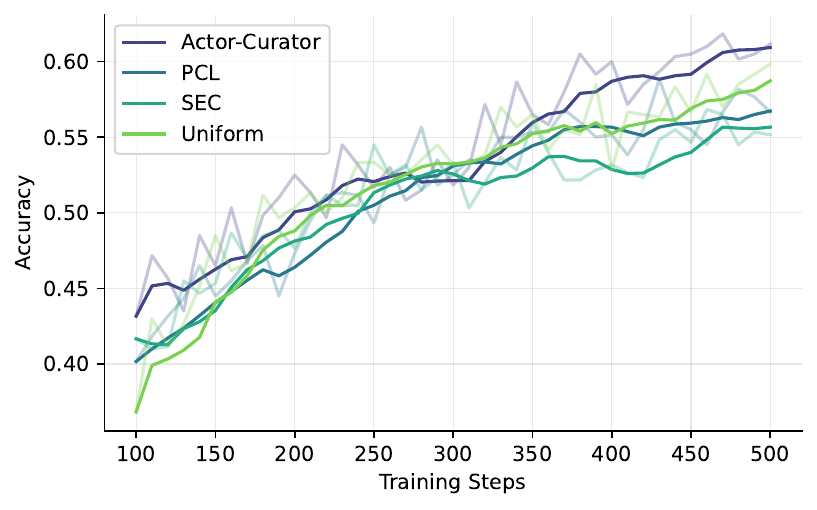}
        \caption{Zebra training curves}
        \label{fig:zebra_training_dynamics}
    \end{subfigure}
    \hfill
    \begin{subfigure}[b]{0.32\textwidth}
        \centering
        \includegraphics[width=\textwidth]{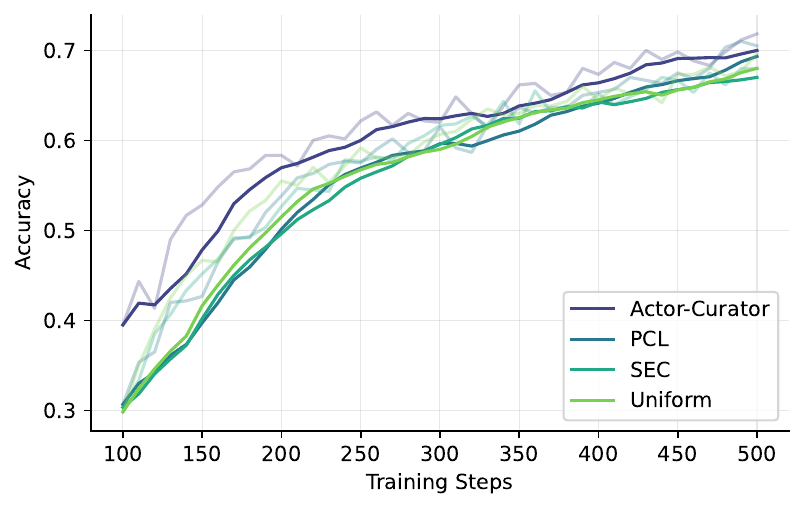}
        \caption{ARC-1D training curves}
        \label{fig:arc_training_dynamics}
    \end{subfigure}
    \caption{\textbf{Training dynamics.}
    \method Test performance over training on three datasets, showing faster convergence and higher final accuracy.}
    \label{fig:training_curves}
    \vspace{-0.1in}
\end{figure*}

\subsection{Ablations}

\paragraph{Core components.}
We ablate two key design choices: the policy-improvement utility and the OSMD bandit objective.
\textbf{Absolute adv} replaces the policy-improvement signal with mean absolute advantage, as in~\citep{chen2025self}.
\textbf{Regression loss} trains the curator to predict the target utility value using a squared loss, rather than learning a classifier as in OSMD. The predicted utilities are then converted into a sampling distribution via a Boltzmann transform with temperature $\eta$, matching the temperature used in OSMD.
As shown in \cref{fig:ab_target}, both components are critical for achieving strong performance. We include a more detailed account of the motivation behind Absolute Adv. and regression in \cref{appendix:baselines}, as well as their difference with \method.

\begin{wrapfigure}{r}{0.45\textwidth}
    \centering

    \includegraphics[width=0.45\textwidth]{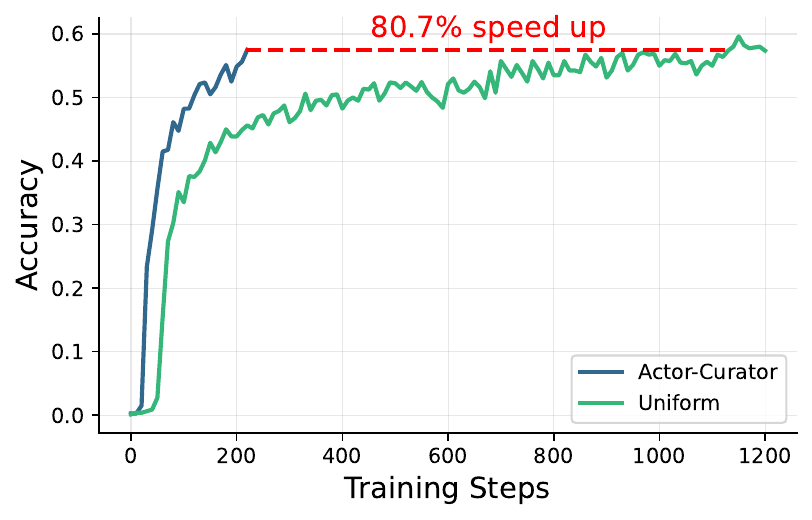}
    \caption{\textbf{Training speed-up.} \method attains high test accuracy with significantly fewer steps.}
    \label{fig:efficiency}

    \vspace{0.1in}

    \includegraphics[width=0.45\textwidth]{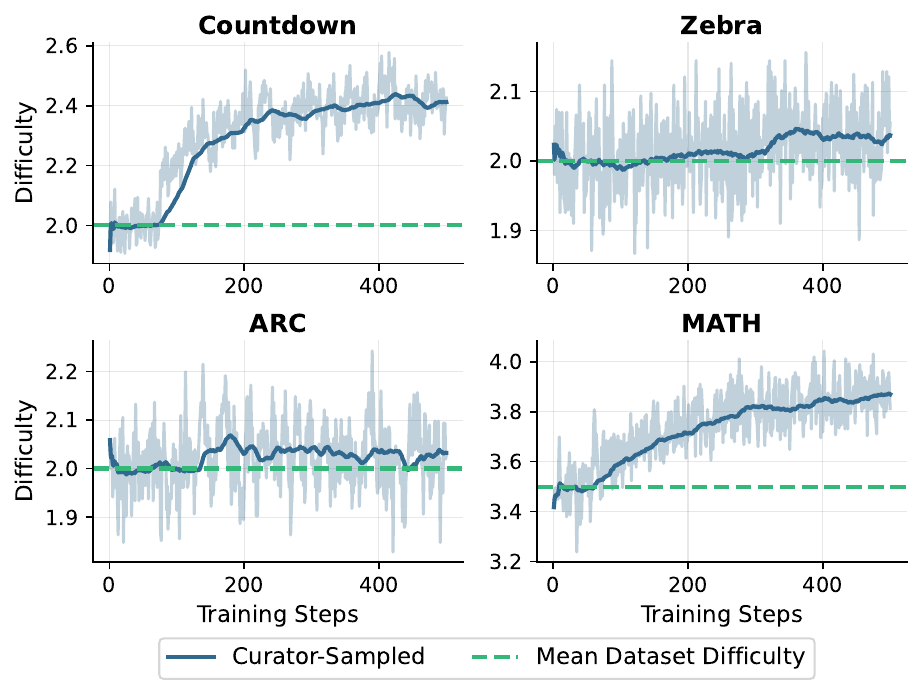}
    \caption{\textbf{Difficulty progression of curated problems.}
    \method gradually increases the average difficulty over training.}
    \label{fig:difficulty_zebra_math}

    \vspace{0.1in}

    \includegraphics[width=0.45\textwidth]{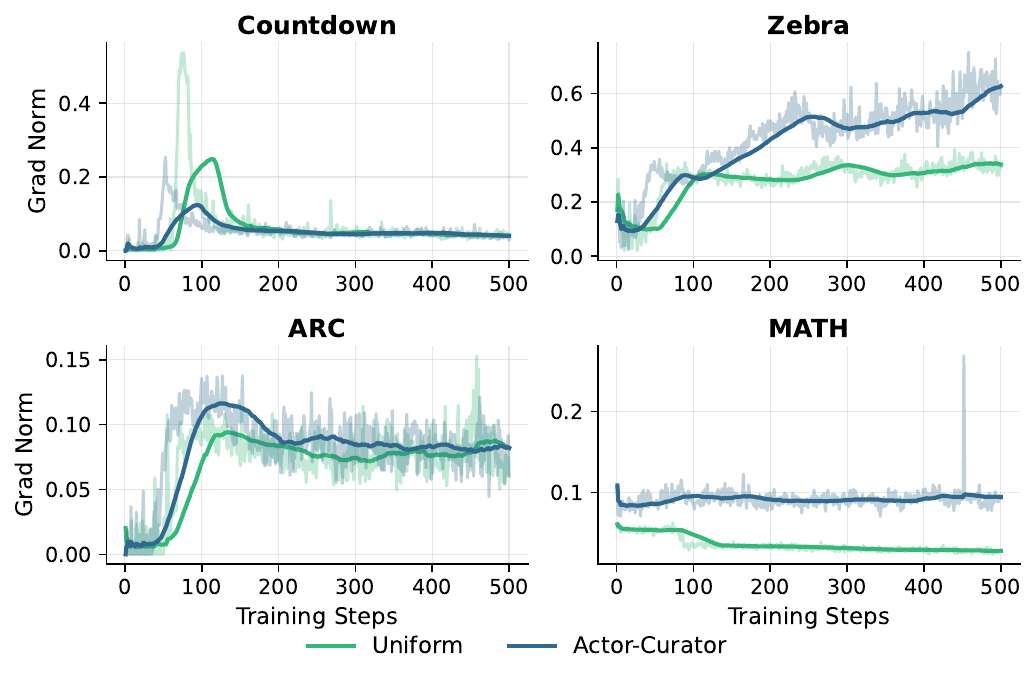}
    \caption{\textbf{Actor gradient norms over training.}
    \method yields larger and more sustained updates.}
    \label{fig:grad_norm_zebra_math}

    \vspace{-1.3in}
\end{wrapfigure}

\paragraph{Actor update methods.}
\method generalizes across actor optimization algorithms.
In \cref{fig:ab_actor_update}, \method significantly improves GRPO-based training relative to uniform sampling, demonstrating robustness to the choice of actor update.

\paragraph{Additional ablation.}
We provide additional ablation on curator model size (\cref{fig:ab_curator_model}) and candidate batch size (\cref{fig:ab_candidate_batch_size}) in \cref{appendix:ablation}.

\subsection{Interpretation and analysis}

\paragraph{Curriculum progression.}
As shown in \cref{fig:difficulty_zebra_math}, \method initially prioritizes easier problems and gradually shifts toward harder ones over training.
Problem difficulty is estimated using heuristic annotations.

\paragraph{Impact on actor updates.}
\method induces consistently larger actor gradient norms than uniform sampling (\cref{fig:grad_norm_zebra_math}), suggesting that curated problems produce more informative learning signals.
This aligns with the curator’s objective of prioritizing problems with higher expected policy improvement. 


\section{Related work}

\paragraph{RLVR.}
Reinforcement learning with verification (RLVR) has emerged as an effective paradigm for improving the capabilities of large language models (LLMs) during post-training~\citep{guo2025deepseek, setlur2024rewarding, chen2025acereason, wen2025light}. Prior work has largely focused on improving the actor update rules~\citep{yu2025dapo, dong2025agentic} or enhancing trajectory generation, often via search-based methods~\citep{zhang2024rest, light2025disc, lightsfs}. Our work is complementary: rather than modifying the actor or rollout process, we focus on learning a principled curriculum that selects which problems the actor should train on to maximize policy improvement.
\vspace{-3mm}

\paragraph{Curriculum learning for RL.}
Classical curriculum learning approaches in reinforcement learning typically select tasks of intermediate difficulty, often defined via the agent’s probability of success~\citep{florensa2017reverse, florensa2018automatic, wohlke2020performance, liu2025understanding}. For example, ProCuRL selects problems whose difficulty is estimated to be near a decision boundary using a learned value network~\citep{tzannetos2023proximal}. In contrast, our approach (1) trains a large language model as a curator that directly operates over language problems, and (2) uses a theoretically grounded, policy-improvement-based target rather than heuristic difficulty estimates. This design allows our method to generalize across different actor update rules and scale to large, heterogeneous datasets.
\vspace{-3mm}

\paragraph{Curriculum learning for LLMs.}
Recent work on self-improving and self-evolving LLMs has highlighted the importance of curriculum learning in RL-based post-training~\citep{ye2024scalable, light2025strategist}. Several methods adjust problem sampling using bandit-style objectives such as UCB~\citep{chen2025self, wang2025dump, gao2025prompt}. However, most existing approaches rely on manual curriculum design, including human-labeled difficulty levels or pre-defined problem buckets that are sampled adaptively~\citep{graves2017automated}. In contrast, we propose one of the first fully automated curriculum learning frameworks for LLM post-training that requires no human annotations or manual dataset structuring.

\begin{figure*}[htbp]
    \centering
    \begin{subfigure}[b]{0.3\textwidth}
        \centering
        \includegraphics[width=\textwidth]{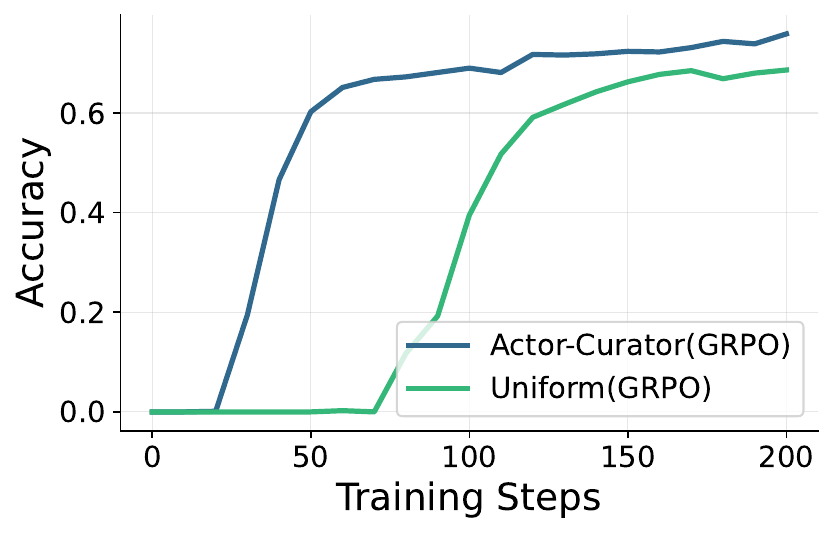}
        \caption{Actor update ablation}
        \label{fig:ab_actor_update}
    \end{subfigure}
    \hfill
    \begin{subfigure}[b]{0.3\textwidth}
        \centering
        \includegraphics[width=\textwidth]{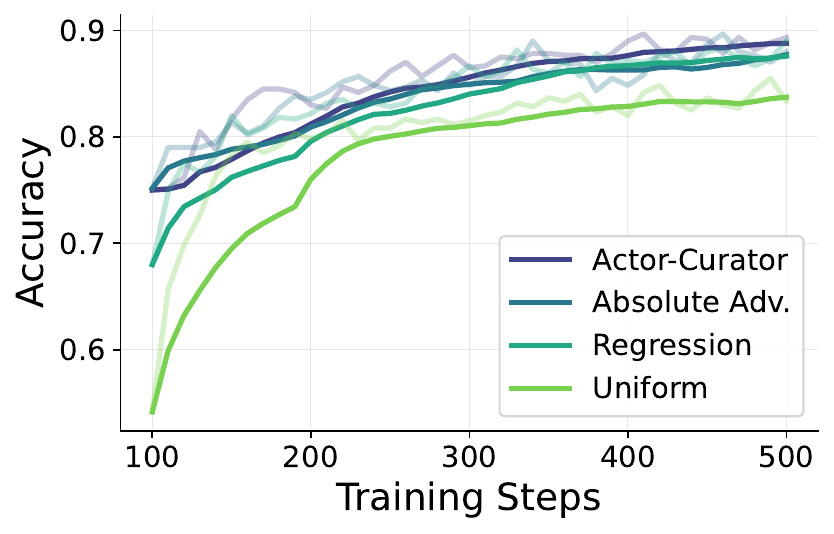}
        \caption{Core component ablation}
        \label{fig:ab_target}
    \end{subfigure}
    \hfill
    \begin{subfigure}[b]{0.3\textwidth}
        \centering
        \includegraphics[width=\textwidth]{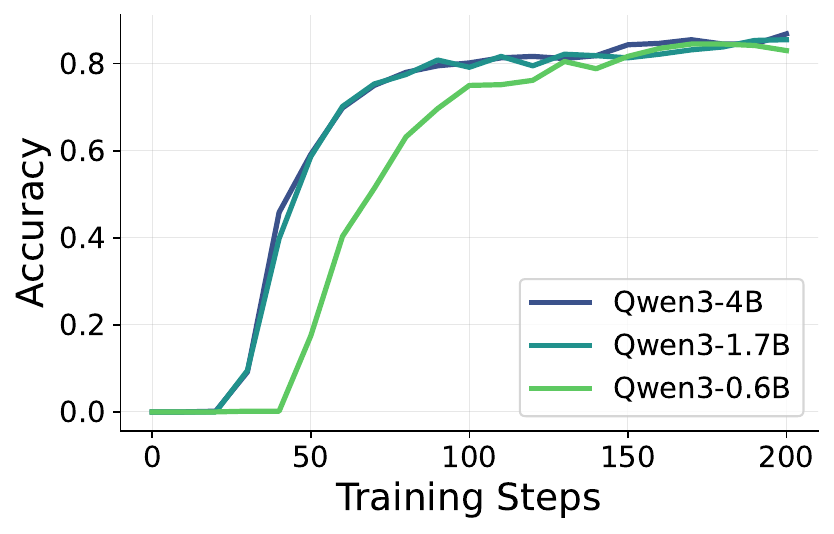}
        \caption{Curator model ablation}
        \label{fig:ab_curator_model}
    \end{subfigure}
    \caption{\textbf{Ablation results.} 
(a) \method is compatible with alternative actor update methods such as GRPO and yields consistent performance gains. 
(b) The combination of the OSMD curator objective and the policy-improvement target achieves superior performance compared to alternative targets and loss functions. 
(c) Varying curator model size leads to similar long-term performance, indicating robustness to curator capabilities.}
\end{figure*}

\section{Limitations}
\label{sec:limitations}
While \method is agnostic to the specific actor update rule and can be combined with a wide range of post-training algorithms, its gains ultimately depend on the stability and quality of the underlying actor optimization: unstable updates can still lead to noisy dynamics or training collapse, a common failure mode in reinforcement learning that curriculum learning alone cannot fully address. In addition, \method inherits standard assumptions from RL-based LLM post-training, most notably access to a reliable reward signal, making it best suited to domains with objectively verifiable rewards such as mathematics, logic puzzles, and code. Finally, learning an explicit curator introduces additional computational overhead beyond standard RL post-training; while modest relative to actor training and rollout generation, this cost is non-negligible, and in our experiments \method increases overall wall-clock training time by approximately 9\% (see \cref{sec:analysis_plus}), though this overhead is small compared to the observed gains in training efficiency.
\vspace{-3mm}

\section{Conclusion}
Our results demonstrate that combining bandit-style optimization, neural function approximation, and policy-improvement-based feedback provides a powerful and general approach to curriculum learning for RL post-training. This combination enables efficient and stable adaptation of training data selection at scale, leading to faster learning and higher final performance, particularly on difficult reasoning problems. Importantly, these gains persist across diverse benchmarks and settings, suggesting that this design generalizes beyond specific tasks or datasets and offers a scalable foundation for improving reinforcement learning post-training of large language models.

Beyond empirical gains, \method highlights a broader shift in how post-training systems should be designed. Rather than treating training data as a static resource, our results suggest that data selection itself can be optimized online as part of the learning process, adapting in tandem with the evolving policy. This perspective opens the door to post-training pipelines that are less reliant on meticulous dataset engineering and more resilient to distributional mismatch, noise, and continual data growth.

\bibliography{ref}
\bibliographystyle{colm2026/colm2026_conference}

\appendix
\newpage
\section{Baselines} \label{appendix:baselines}

\subsection{Group-Based Policy Optimization} \label{appendix:baselines_group}

Group-based policy optimization methods update the policy by grouping multiple rollouts from the same problem and computing advantages relative to the group baseline.

\paragraph{GRPO.} Group Relative Policy Optimization (GRPO) \citep{shao2024deepseekmath} computes advantages by comparing each solution's reward to the mean reward of all solutions sampled from the same problem. For a problem $x$ with rollouts $\mathcal{Y}_x = \{y^{(1)}, \ldots, y^{(m)}\}$ and corresponding rewards $\{R(y^{(j)} | x)\}_{j=1}^m$, the advantage for solution $y^{(i)}$ is:
\begin{equation}
A^{\text{GRPO}}(y^{(i)} | x) = R(y^{(i)} | x) - \frac{1}{m}\sum_{j=1}^m R(y^{(j)} | x).
\end{equation}

The policy is then updated using a PPO-style clipped objective:
\begin{equation}
\mathcal{L}^{\text{GRPO}}(\pi) = \mathbb{E}_{x, y^{(i)} \sim \mathcal{Y}_x} \left[ \min\left( \rho_i A^{\text{GRPO}}(y^{(i)} | x), \text{clip}(\rho_i, 1-\epsilon, 1+\epsilon) A^{\text{GRPO}}(y^{(i)} | x) \right) \right],
\end{equation}
where $\rho_i = \frac{\pi(y^{(i)} | x)}{\pi_{\text{old}}(y^{(i)} | x)}$ is the importance ratio and $\epsilon$ is the clipping threshold.

\paragraph{GSPO.} Group-Sequence Policy Optimization \citep{zheng2025group} addresses fundamental stability issues in GRPO by defining importance ratios at the sequence level rather than the token level. Unlike GRPO, which applies token-level importance weights that can introduce high-variance noise, GSPO computes importance ratios based on sequence likelihood, aligning with the principle of importance sampling.

GSPO optimizes the following sequence-level objective:
\begin{equation}
\mathcal{J}^{\text{GSPO}}(\theta) = \mathbb{E}_{x \sim \mathcal{D}, \{y_i\}_{i=1}^G \sim \pi_{\theta_{\text{old}}}(\cdot | x)} \left[ \frac{1}{G} \sum_{i=1}^G \min\left( s_i(\theta) \hat{A}_i, \text{clip}(s_i(\theta), 1-\epsilon, 1+\epsilon) \hat{A}_i \right) \right],
\end{equation}

where the sequence-level importance ratio is defined as:
\begin{equation}
s_i(\theta) = \left( \frac{\pi_\theta(y_i | x)}{\pi_{\theta_{\text{old}}}(y_i | x)} \right)^{\frac{1}{|y_i|}} = \exp\left( \frac{1}{|y_i|} \sum_{t=1}^{|y_i|} \log \frac{\pi_\theta(y_{i,t} | x, y_{i,<t})}{\pi_{\theta_{\text{old}}}(y_{i,t} | x, y_{i,<t})} \right).
\end{equation}
The advantage computation remains the same as in GRPO.

\subsection{Mean Absolute Advantage as Curriculum Reward} \label{appendix:baselines_mean_abs_adv}
An ideal curriculum should prioritize training problems that maximize the model's immediate learning outcomes. A natural way to quantify learning outcomes is through the magnitude of parameter updates induced by the selected training data, which can be approximated by the absolute advantage.

In the common setting of RL with verifiable binary rewards, mean absolute advantage has a attractive interpretation. When using group-based RL methods like GRPO with $n$ rollouts per problem, the advantage for the $i$-th rollout is computed as:
\begin{equation}
\hat{A}_{t,i} = \frac{r_i - \text{mean}(r)}{\text{std}(r)} = \frac{r_i - p}{\sqrt{p(1-p)}},
\end{equation}
where $p$ is the empirical success rate over the group. Since the reward $r_i$ follows a Bernoulli distribution, the expected absolute advantage is:
\begin{align}
\mathbb{E}[|\hat{A}_{t,i}|] &= p \cdot \frac{1-p}{\sqrt{p(1-p)}} + (1-p) \cdot \frac{p}{\sqrt{p(1-p)}} = 2\sqrt{p(1-p)}.
\end{align}

The function $g(p) = 2\sqrt{p(1-p)}$ is symmetric around $p = 0.5$, strictly concave on $[0,1]$, and reaches its maximum at $p = 0.5$. Thus, maximizing expected absolute advantage is equivalent to prioritizing problems at a success rate of 50\%.

\subsection{Value-Based Curation} \label{appendix:baselines_value_based}

Value-based curation methods maintain explicit estimates of the utility or learning value associated with training data, which are then used to guide adaptive curriculum selection during RL post-training.

The core principle of value-based curation is to learn a utility function $Q_t$ that maps training data to expected learning outcomes at training step $t$. Let $\mathcal{S}$ denote the space over which utilities are estimated (e.g., individual problems, problem categories, or problem features). The utility function $Q_t: \mathcal{S} \rightarrow \mathbb{R}$ assigns a scalar value to each element $s \in \mathcal{S}$.

During curriculum selection, the curator samples from $\mathcal{S}$ according to a policy derived from $Q_t$. A common choice is the Boltzmann (softmax) policy:
\begin{equation}
p_t(s) = \frac{\exp(Q_t(s)/\eta)}{\sum_{s' \in \mathcal{S}} \exp(Q_t(s')/\eta)},
\end{equation}
where $\eta > 0$ is a temperature parameter that controls the exploration-exploitation tradeoff: higher temperatures lead to more uniform sampling (exploration), while lower temperatures concentrate probability mass on high-utility items (exploitation).

The utility function is updated over time based on observed learning outcomes. After selecting data according to $p_t$ and performing an RL update, the curator observes a reward signal $r_t(s)$ that measures the actual learning value obtained. These observations are used to refine $Q_{t+1}$.

\paragraph{SEC.}
Self-Evolving Curriculum (SEC) \citep{chen2025self} instantiates value-based curation by defining the utility space $\mathcal{S} = C$ as a discrete set of problem categories (e.g., difficulty levels, problem types). The utility function $Q_t: C \rightarrow \mathbb{R}$ is represented as a lookup table, with one scalar value per category.

SEC uses temporal difference (TD) learning to update utilities:
\begin{equation}
Q_{t+1}(c) = \alpha r_t(c) + (1 - \alpha) Q_t(c),
\end{equation}
where $\alpha \in (0, 1]$ is a learning rate and $r_t(c)$ is the mean absolute advantage aggregated over all problems from category $c$ selected at step $t$. This exponential moving average naturally adapts to non-stationarity as category utilities change with model improvement.

Curriculum selection proceeds in two stages: first, a category is sampled according to the Boltzmann policy $p_t(c) \propto \exp(Q_t(c)/\eta)$; then, problems are uniformly sampled from the selected category.

\textbf{Regression.}
An alternative instantiation defines the utility space $\mathcal{S} = \mathcal{X}$ directly at the problem level, estimating utilities for individual training problems. The utility function is parameterized by a neural network $Q_\phi: \mathcal{X} \rightarrow \mathbb{R}$ that takes problem representations (e.g., text embeddings) as input.

Rather than incremental TD updates, the curator is trained via supervised regression. Let $\mathcal{H}_t = \{(x_j, r_j)\}_{j=1}^t$ denote the history of problem-reward observations. The utility network is optimized to minimize mean squared error:
\begin{equation}
\mathcal{L}_{\text{MSE}}(\phi) = \sum_{(x_j, r_j) \in \mathcal{H}_t} \left(Q_\phi(x_j) - r_j\right)^2.
\end{equation}

This can be optimized periodically (every $K$ steps) via batch gradient descent, or online via stochastic gradient descent:
\begin{equation}
\phi_{t+1} = \phi_t - \beta \nabla_\phi \left(Q_\phi(x_t) - r_t(x_t)\right)^2,
\end{equation}
where $\beta$ is the learning rate. Problems are sampled directly according to $p_t(x) \propto \exp(Q_\phi(x)/\eta)$.

\textbf{PCL} \citep{gao2025prompt} exemplifies the online regression approach. PCL updates $Q_\phi$ concurrently with policy training using only the current batch of observations. At each step $t$, PCL samples a candidate pool of $km$ prompts and selects the $m$ prompts whose predicted values are closest to a target threshold $\tau$ (typically 0.5):
\begin{equation}
D_m = \arg\min_{S \subseteq D_{km}, |S|=m} \sum_{x \in S} |Q_\phi(x) - \tau|.
\end{equation}
This greedy selection can be viewed as an extreme limit of Boltzmann sampling with a sharply peaked distribution. Consider the modified utility $\tilde{Q}(x) = -|Q_\phi(x) - \tau|$, which measures negative distance from the threshold. As the temperature $\eta \to 0$, the Boltzmann policy $p(x) \propto \exp(\tilde{Q}(x)/\eta)$ concentrates all probability mass on prompts nearest to $\tau$, recovering PCL's greedy selection. This deterministic selection strategy is computationally efficient and ensures the training batch contains only prompts of target difficulty, maximizing the effective ratio of informative gradient signals.

After selecting prompts, PCL generates $n$ responses per prompt and updates both the policy and value model. The value model is trained on the observed rewards from the selected prompts:
\begin{equation}
\mathcal{L}_{\text{PCL}}(\phi) = \sum_{i=1}^m \left(Q_\phi(x_i) - \frac{1}{n}\sum_{j=1}^n r(x_i, y_{i,j})\right)^2.
\end{equation}

\paragraph{Comparison to Actor-Curator.}
The Actor-Curator framework with OSMD differs fundamentally from value-based curation methods in its optimization objective and data selection mechanism. While value-based methods learn utilities $Q_t(s)$ to predict expected learning outcomes and sample accordingly, OSMD directly optimizes a curriculum distribution $\boldsymbol{p}_t$ to maximize expected policy improvement. Value-based approaches require estimating problem-level or category-level values and making selection decisions based on these estimates—a two-stage process that introduces approximation error. In contrast, OSMD treats curriculum optimization as a first-order problem: the gradient $\nabla_q J(\pi_t, \boldsymbol{p})$ directly specifies how to adjust the data distribution to improve the policy, without requiring intermediate value estimates. Furthermore, value-based methods typically rely on scalar reward signals $r_t(s)$ to update utilities, whereas OSMD leverages the full gradient information $\nabla_\theta \mathcal{L}(\pi_t; x)$ to measure the learning value of each problem. This allows OSMD to capture richer information about how individual problems affect policy optimization, beyond what a single scalar reward can convey.

\newpage

\section{Policy improvement estimation}
\label{sec:performance_bound}

This section proves unbiasedness of the per-problem bandit feedback estimators used to train the curator.
Recall the per-problem utility at iteration $t$ (Eq.~\eqref{eq:per_problem_utility}):
\[
u_{\boldsymbol{x}}^t
=
p_{\mathcal{X}}(\boldsymbol{x})\,
\EE_{\boldsymbol{y}\sim \pi^{t}(\cdot\mid\boldsymbol{x})}
\left[
\frac{\pi^{t+1}(\boldsymbol{y}\mid\boldsymbol{x})}{\pi^{t}(\boldsymbol{y}\mid\boldsymbol{x})}
A_{\pi^t}(\boldsymbol{y}\mid\boldsymbol{x})
\right].
\]
We also recall the rollout-based estimator (Eq.~\eqref{eq:x_contribution_estimated}):
\[
\hat{A}^t(\cdot \mid \boldsymbol{x})
\triangleq
\frac{1}{|\mathcal{Y}^t_{\boldsymbol{x}}|}
\sum_{\boldsymbol{y} \in \mathcal{Y}^t_{\boldsymbol{x}}}
\frac{\pi^{t+1}(\boldsymbol{y}\mid \boldsymbol{x})}{\pi^{t}(\boldsymbol{y}\mid \boldsymbol{x})}
A_{\pi^t}(\boldsymbol{y}\mid \boldsymbol{x}),
\quad
\mathcal{Y}^t_{\boldsymbol{x}} = \{\boldsymbol{y}^{(j)}\sim \pi^t(\cdot\mid \boldsymbol{x})\}_{j=1}^{|\mathcal{Y}_{\boldsymbol{x}}|}.
\]

\subsection{Single-stage case}
\label{sec:single_stage_unbiased}

\begin{theorem}[Unbiasedness (single-stage)]
\label{thm:unbiased_single_stage}
For the estimator in Eq.~\eqref{eq:x_contribution_estimated},
\[
\hat{U}^t_{\boldsymbol{x}}
=
p_\mathcal{X}(\boldsymbol{x})
\frac{\mathbb{I}\{\boldsymbol{x} \in \mathcal{X}^t\}}{p^t(\boldsymbol{x})} 
\hat{A}^t(\cdot \mid \boldsymbol{x}),
\]
we have $\EE[\hat{U}^t_{\boldsymbol{x}}]=u^t_{\boldsymbol{x}}$ for every $\boldsymbol{x}\in\mathcal{X}$.
\end{theorem}

\begin{proof}
Fix an iteration $t$ and a problem $\boldsymbol{x}\in\mathcal{X}$.
In the single-stage setting, the training set $\mathcal{X}^t$ is sampled directly from $\mathcal{X}$
according to the curator distribution $p^t(\cdot)$, so that
\[
\Pr(\boldsymbol{x}\in \mathcal{X}^t)=p^t(\boldsymbol{x}).
\]

First, $\hat{A}^t(\cdot\mid \boldsymbol{x})$ is an unbiased estimator of the population quantity
\[
\EE_{\boldsymbol{y}\sim \pi^{t}(\cdot\mid\boldsymbol{x})}
\left[
\frac{\pi^{t+1}(\boldsymbol{y}\mid\boldsymbol{x})}{\pi^{t}(\boldsymbol{y}\mid\boldsymbol{x})}
A_{\pi^t}(\boldsymbol{y}\mid\boldsymbol{x})
\right],
\]
by i.i.d.\ rollout sampling and linearity of expectation.

Next, take expectation of $\hat{U}^t_{\boldsymbol{x}}$ conditioning on $\hat{A}^t(\cdot\mid \boldsymbol{x})$:
\[
\EE\!\left[\hat{U}^t_{\boldsymbol{x}} \mid \hat{A}^t(\cdot\mid \boldsymbol{x})\right]
=
p_{\mathcal X}(\boldsymbol{x})\,
\EE\!\left[\frac{\mathbb{I}\{\boldsymbol{x}\in\mathcal{X}^t\}}{p^t(\boldsymbol{x})}\right]\,
\hat{A}^t(\cdot\mid \boldsymbol{x})
=
p_{\mathcal X}(\boldsymbol{x})\,
\hat{A}^t(\cdot\mid \boldsymbol{x}),
\]
since $\EE[\mathbb{I}\{\boldsymbol{x}\in\mathcal{X}^t\}]=\Pr(\boldsymbol{x}\in\mathcal{X}^t)=p^t(\boldsymbol{x})$.

Finally, taking expectation over rollout randomness yields
\[
\EE[\hat{U}^t_{\boldsymbol{x}}]
=
p_{\mathcal X}(\boldsymbol{x})\,
\EE\!\left[\hat{A}^t(\cdot\mid \boldsymbol{x})\right]
=
u^t_{\boldsymbol{x}},
\]
which proves the claim.
\end{proof}

\subsection{Two-stage case}
\label{sec:two_stage_unbiased}

\begin{theorem}[Unbiasedness (two-stage)]
\label{thm:unbiased_two_stage}
For the two-stage estimator in Eq.~\eqref{eq:x_contribution_estimated_two_stage},
\[
\hat{U}^{t}_{\mathrm{two}, \boldsymbol{x}}
=
p_\mathcal{X}(\boldsymbol{x})\,
\frac{\mathbb{I}\{\boldsymbol{x} \in \mathcal{X}^t\}}
{q(\boldsymbol{x})\,p^t(\boldsymbol{x} \mid \tilde{\mathcal{X}}^t)}
\;\hat{A}^t(\cdot \mid \boldsymbol{x}),
\]
we have $\EE[\hat{U}^{t}_{\mathrm{two}, \boldsymbol{x}}]=u^t_{\boldsymbol{x}}$ for every $\boldsymbol{x}\in\mathcal{X}$.
\end{theorem}

\begin{proof}
Fix an iteration $t$ and a problem $\boldsymbol{x}\in\mathcal{X}$.
By definition of the two-stage procedure (Section~\ref{sec:two_stage_sampling}),
the candidate set $\tilde{\mathcal{X}}^t$ is sampled from $\tilde{q}$, inducing the marginal inclusion probability
$q(\boldsymbol{x})=\Pr(\boldsymbol{x}\in\tilde{\mathcal{X}}^t)$.
Conditioned on $\tilde{\mathcal{X}}^t$, the curator selects the training set $\mathcal{X}^t\subset \tilde{\mathcal{X}}^t$
according to $p^t(\cdot\mid \tilde{\mathcal{X}}^t)$.

As in the single-stage case, $\hat{A}^t(\cdot\mid \boldsymbol{x})$ is an unbiased estimator of the corresponding population
expectation under $\boldsymbol{y}\sim \pi^t(\cdot\mid\boldsymbol{x})$.

Now condition on the realized candidate set $\tilde{\mathcal{X}}^t$ and on $\hat{A}^t(\cdot\mid \boldsymbol{x})$.
If $\boldsymbol{x}\notin \tilde{\mathcal{X}}^t$, then $\mathbb{I}\{\boldsymbol{x}\in \mathcal{X}^t\}=0$ almost surely.
If $\boldsymbol{x}\in \tilde{\mathcal{X}}^t$, then by the selection step,
\[
\EE\!\left[\mathbb{I}\{\boldsymbol{x}\in \mathcal{X}^t\}\mid \tilde{\mathcal{X}}^t\right]
=
p^t(\boldsymbol{x}\mid \tilde{\mathcal{X}}^t),
\]
and therefore
\[
\EE\!\left[
\frac{\mathbb{I}\{\boldsymbol{x}\in \mathcal{X}^t\}}{p^t(\boldsymbol{x}\mid \tilde{\mathcal{X}}^t)}
\;\middle|\;
\tilde{\mathcal{X}}^t, \hat{A}^t(\cdot\mid \boldsymbol{x})
\right]
=
\mathbb{I}\{\boldsymbol{x}\in \tilde{\mathcal{X}}^t\}.
\]
Substituting into the estimator gives
\[
\EE\!\left[
\hat{U}^{t}_{\mathrm{two}, \boldsymbol{x}}
\;\middle|\;
\tilde{\mathcal{X}}^t, \hat{A}^t(\cdot\mid \boldsymbol{x})
\right]
=
p_{\mathcal X}(\boldsymbol{x})\,
\frac{\mathbb{I}\{\boldsymbol{x}\in \tilde{\mathcal{X}}^t\}}{q(\boldsymbol{x})}
\;\hat{A}^t(\cdot\mid \boldsymbol{x}).
\]
Taking expectation over the proposal step yields
\[
\EE\!\left[
\hat{U}^{t}_{\mathrm{two}, \boldsymbol{x}}
\;\middle|\;
\hat{A}^t(\cdot\mid \boldsymbol{x})
\right]
=
p_{\mathcal X}(\boldsymbol{x})\,
\frac{\EE[\mathbb{I}\{\boldsymbol{x}\in \tilde{\mathcal{X}}^t\}]}{q(\boldsymbol{x})}
\;\hat{A}^t(\cdot\mid \boldsymbol{x})
=
p_{\mathcal X}(\boldsymbol{x})\,
\hat{A}^t(\cdot\mid \boldsymbol{x}),
\]
since $\EE[\mathbb{I}\{\boldsymbol{x}\in \tilde{\mathcal{X}}^t\}]=\Pr(\boldsymbol{x}\in \tilde{\mathcal{X}}^t)=q(\boldsymbol{x})$.
Finally, taking expectation over rollout randomness gives
\[
\EE[\hat{U}^{t}_{\mathrm{two}, \boldsymbol{x}}]
=
p_{\mathcal X}(\boldsymbol{x})\,
\EE[\hat{A}^t(\cdot\mid \boldsymbol{x})]
=
u^t_{\boldsymbol{x}},
\]
which proves the claim.
\end{proof}

\newpage
\section{Idealized OSMD algorithm} 
\label{sec:idealized-omd}

\begin{algorithm}
\caption{Sleeping Online Mirror Descent}
\label{alg:somd}
\begin{algorithmic}[1]
\Require Number of total arms $K$, number of available arms $k$ each round, horizon $T$, step size $\eta>0$, exploration parameter $\alpha\in(0, 1/k)$.
\State Initialize $\boldsymbol{p}_1 = (1/K, \dots, 1/K)$.
\For{$t=1,2,\dots,T$}
    \State Sample available subset $\Tilde{\mathcal{X}}_t$. Compute 
    \begin{equation} \label{eqn:Q_t}
    \boldsymbol{p}^t(i \mid \Tilde{\mathcal{X}}_t) = \begin{cases}
        \frac{\boldsymbol{p}_{t, i}}{\sum_{j \in \Tilde{\mathcal{X}}_t} \boldsymbol{p}_{t, j}} & \text{if $i \in \Tilde{\mathcal{X}}_t$} \\
        0 \quad & \text{otherwise}
    \end{cases}
    \end{equation}
    \State Sample $a_{t, 1}, a_{t, 2}, \dots, a_{t, s} \stackrel{i.i.d.}{\sim} \boldsymbol{p}^t(\cdot \mid \Tilde{\mathcal{X}}_t)$
    \State \textbf{(Loss Estimator)} For each arm $i\in[K]$, set
    \begin{equation} \label{eqn:somd-estimator}
    \begin{aligned}
    \widehat{L}_{t, i} = \frac{1}{s} \sum_{r=1}^s \mathbf{1}\{a_{t, r} = i\} \frac{l_{t, i}}{\boldsymbol{p}^t(i \mid \Tilde{\mathcal{X}}_t)} \\
    \end{aligned}
    \end{equation}
    \State \textbf{(OSMD Update)} Update the next distribution by the mirror step
    \begin{equation} \label{eqn:somd-update}
    \begin{aligned}
    &\boldsymbol{p}_{t+1} \in \arg\min_{p\in\mathcal{A}}\left\{ \eta \,\langle p, \widehat{\boldsymbol{L}}_{t}\rangle + D_F(p, \boldsymbol{p}_t)\right\}, \\ 
    &\text{where} \quad D_F(\boldsymbol{u}, \boldsymbol{v}) = \sum_i u_i \log\frac{u_i}{v_i}
    \end{aligned}
    \end{equation}
\EndFor
\end{algorithmic}
\end{algorithm}

\subsection{Setup}
We formalize the tabular bandit algorithm in Section~\ref{sec:methodology} as Algorithm~\ref{alg:somd}. Under this idealized algorithm. We assume there is a large set of $K$ arms. At each time step $t$, $k$ arms are uniformly randomly chosen as candidate arms. The settings allows for the pulling of $s$ arms per round, and after each round the loss $l_{t, i}$ is revealed for each chosen arm $i$, where $i$ indicates the arm's original index in $[K]$. Throughout this section and the next, we denote the vectorized quantities with bold font. For example $\boldsymbol{l}_t = (l_{t, 1}, \dots, l_{t, K})$.

Algorithm~\ref{alg:somd} departs from traditional Online Stochastic Mirror Descent (OSMD) due to the availability constraint. The action distribution is conditioned on a randomly sampled \emph{available set} at each round, and the loss estimator is modified to remain unbiased under this conditional sampling.

Eqn.~\eqref{eqn:Q_t} introduces a two-stage sampling process. First, a random $k$-subset $\Tilde{\mathcal{X}}_t \subseteq [K]$ of available arms is drawn uniformly. The learner then constructs a \emph{conditional distribution}
\[
\boldsymbol{p}^t(i \mid \Tilde{\mathcal{X}}_t)
= \frac{\boldsymbol{p}_{t,i}}{\sum_{j \in \Tilde{\mathcal{X}}_t} \boldsymbol{p}_{t,j}}
\quad \text{for } i \in \Tilde{\mathcal{X}}_t,
\]
and assigns zero probability to arms outside $\Tilde{\mathcal{X}}_t$. We sometimes also use $\boldsymbol{p}_{\mid \Tilde{\mathcal{X}}_t}(i)$ to denote $\boldsymbol{p}^t(i \mid \Tilde{\mathcal{X}}_t)$. This renormalization ensures that the learner only samples from arms that are available at round $t$, while still using $\boldsymbol{p}_t$ as the global state variable that is updated over time.

The conditional sampling in Eqn.~\eqref{eqn:Q_t} invalidates the standard OSMD estimator, since $p_{t,i}$ is no longer the actual probability with which arm $i$ is sampled. Eqn.~\eqref{eqn:somd-estimator} addresses this by defining
\[
\widehat{L}_{t,i}
= \frac{1}{s}\sum_{r=1}^s \mathbf{1}\{a_{t,r}=i\}
\frac{l_{t,i}}{\boldsymbol{p}^t(i \mid \Tilde{\mathcal{X}}_t)}.
\]
This estimator uses the conditional probability $\boldsymbol{p}^t(i \mid \Tilde{\mathcal{X}}_t)$ in the denominator, which is the true sampling probability of arm $i$ given the realized availability set. As a result, conditional on $\Tilde{\mathcal{X}}_t$, the estimator is unbiased:
\[
\mathbb{E}\!\left[\widehat{L}_{t,i} \mid \Tilde{\mathcal{X}}_t\right]
= l_{t,i}\,\mathbb{I}\{i \in \Tilde{\mathcal{X}}_t\}.
\]
The use of $s$ independent samples further reduces variance and corresponds to a semi-bandit feedback model, but does not change the role of the estimator in the mirror update.

\section{Regret Analysis} \label{sec:regret-proof}
This section presents a proof of the regret bound (Theorem~\ref{thm:regret}) for Algorithm \ref{alg:somd}.
Suppose there are $K$ base arms $[K]=\{1, \dots, K\}$ each corresponding to a problem in the dataset $\mathcal{X} = \{\boldsymbol{x}^{(1)}, \boldsymbol{x}^{(2)}, \dots, \boldsymbol{x}^{(K)}\}$. At each step, a subset $\Tilde{\mathcal{X}}_t$ of size $k$ is drawn uniformly randomly from $\mathcal{X}$. For every round $t$, The curator picks $s$ arms $a_{t, 1}, a_{t, 2}, \dots, a_{t, s} \in [K]$ from $\Tilde{\mathcal{X}}_t$ and the loss for each arm $l_{t, a_{t, i}}$ is revealed. For convenience, write $\boldsymbol{l}_t = (l_{t, 1}, \dots, l_{t, K})$ to be the losses of each arm as a vector. Assume without loss of generality that $l_{t, i} \in [0, 1]$ for all $t$ and $i$. Define the subset-masked loss vector $\boldsymbol{l}_t^{\Tilde{\mathcal{X}}_t} \in \mathbb{R}^K$:
\begin{equation}
\boldsymbol{l}_t^{\Tilde{\mathcal{X}}_t}(i) = \begin{cases}
    \boldsymbol{l}_t(i) & \text{if } i \in \Tilde{\mathcal{X}}_t \\
    0 & \text{otherwise}
\end{cases}
\end{equation}
Denote the best available arm at time $t$ as $m_t$
\begin{equation}
m_t = \arg\min_{i \in \Tilde{\mathcal{X}}_t} l_{t, i} \text{ and } l^*_t = l_{t, m_t}.
\end{equation}
We define the \textit{best-arm regret} to be
\begin{equation}
\mathrm{Reg}_T^{\mathrm{best}} = \EE\left[ \sum_{t=1}^T (\sum_{i=1}^s l _{t, a_{t, s}} - l^*_t) \right]
\end{equation}
Note that the regret can be expressed in the vectorized form.
\begin{equation} \label{eqn:regret-best}
\mathrm{Reg}_n^{\mathrm{best}} = \EE\left[ \sum_{t=1}^T\left\langle \boldsymbol{p}_{\Tilde{\mathcal{X}}_t}, \boldsymbol{l}_t \right\rangle - \left\langle \boldsymbol{e}_{m_t}, \boldsymbol{l}_t \right\rangle  \right],
\end{equation}
where $\boldsymbol{e}_i$ is a one-hot vector with $1$ on index $i$. Here, we use $\boldsymbol{p}_{\mid \Tilde{\mathcal{X}}_t}$ to represent the vector $(\boldsymbol{p}_{\mid \Tilde{\mathcal{X}}_t}\boldsymbol{x}^{(1)}), \boldsymbol{p}_{\mid \Tilde{\mathcal{X}}_t}(\boldsymbol{x}^{(2)}), \dots, \boldsymbol{p}_{\mid \Tilde{\mathcal{X}}_t}(\boldsymbol{x}^{(K)}))$. In practice, forcing uniform exploration usually have negligible or even positive effect on the performance. However, it incurs linear regret for theoretical analysis. Thus, to focus on how OSMD manages losses in a non-stationary environment, this regret analysis isolate the loss attributable to factors other than uniform exploration. Towards this end, we set the comparator $\boldsymbol{q}_t$ to be a mixture of the best available arm and uniform distribution over all arms.
\begin{equation}
    \boldsymbol{q}_t = (1 - k\alpha)\boldsymbol{e}_{m_t} + \alpha \boldsymbol{1}_{\Tilde{\mathcal{X}}_t}
\end{equation}
where $\alpha > 0$ and $\mathbf{1}_{\Tilde{\mathcal{X}}_t} \in \mathbb{R}^K$ is a binary vector with $\mathbf{1}_{\Tilde{\mathcal{X}}_t}(i) = 1$ for all $i \in \Tilde{\mathcal{X}}_t$ and zero everywhere else. The \textit{best-available regret} is defined as
\begin{equation} \label{eqn:smoothed-reg}
\mathrm{Reg}_T^{\mathrm{BA}} = \EE\left[ \sum_{t=1}^T \left\langle \boldsymbol{p}_{\mid \Tilde{\mathcal{X}}_t}, \boldsymbol{l}_t \right\rangle - \left\langle \boldsymbol{q}_t, \boldsymbol{l}_t \right\rangle  \right]
\end{equation}

\begin{theorem}[Restatement of Theorem~\ref{thm:regret}] \label{thm:regret-simplified}
Let the \textit{drift} parameter
\begin{equation} \label{eqn:drift-defn}
V_T = \sum_{t=2}^T \Delta_t, \quad \Delta_t = \max_{i \in [K]} \left| l_{t, i} - l_{t-1, i} \right|
\end{equation}
be a measure of how rapid the arm values changes over time $t$. Assume without loss of generality that $\boldsymbol{l}_{t, i} \in [0, 1]$ for all $t$ and $i$. Then we have
\begin{equation} \label{eqn:regret-simplified}
\mathrm{Reg}^{\mathrm{BA}}_n
\le
O\left(T^{2/3} V_n^{1/3}\right)
\end{equation}
for Algorithm~\ref{alg:somd}.
\end{theorem}

To bound the smoothed regret, the proof follows a block decomposition argument. We partition the horizon into contiguous blocks of length $B$ and restart the algorithm at the beginning of each block. Within a single block, the loss sequence is treated as approximately stationary, allowing us to compare the algorithm against a fixed comparator using standard OSMD analysis. The regret within each block is controlled by a stability--variance tradeoff: the mirror descent inequality yields a term of order $O(\log(1/\alpha)/\eta)$, while the variance of the importance-weighted estimator contributes a term proportional to $O(\eta B)$. Across blocks, non-stationarity is captured through a variation budget $V_n$, which upper bounds the cumulative discrepancy between the true per-round losses and the frozen losses used in the blockwise analysis, contributing an additive term of order $B V_n$. The block length $B$ is then optimized to balance the statistical cost of restarting too frequently against the bias induced by treating losses as stationary within each block. This optimization yields a regret bound scaling as $O(V_n^{1/3})$, reflecting the classical tradeoff between adaptivity to non-stationarity and estimation error \cite{non-stationary-bandit}.

Let the length of each aforementioned block be $L$. Denote the start time of each block as $\tau_l = (l - 1)L + 1$, where $l = 1, 2, \dots,  B$ and $B = T/L$\footnote{For convenience, we assume $T$ is divisible by B and L at the same time.}. Denote the time steps of the $l$-th block as $\mathcal{I}_l = \{\tau_l, \dots, \tau_l + L - 1\}$.
\begin{equation} \label{eqn:smoothed-reg-decomp}
\mathrm{Reg}_n^{\mathrm{BA}} = \EE\left[ \sum_{t=1}^n \left\langle \boldsymbol{p}_{\mid \Tilde{\mathcal{X}}_t}, \boldsymbol{l}_t \right\rangle - \left\langle \boldsymbol{q}_t, \boldsymbol{l}_t \right\rangle  \right] = \sum_{l=1}^B \EE \left[ \sum_{t \in \mathcal{I}_l} \left\langle \boldsymbol{p}_{\mid \Tilde{\mathcal{X}}_t}, \boldsymbol{y}_t \right\rangle - \left\langle \boldsymbol{q}_t, \boldsymbol{l}_t \right\rangle \right]
\end{equation}

\subsection{Time-Frozen Comparator}
This part of the proof introduces a stable reference for comparison in the presence of non-stationary losses. Rather than comparing the learner to the best action at every round, which may change arbitrarily over time, we define a comparator based on the losses at a fixed reference $\tau$. This ``time-frozen'' comparator serves as a proxy for the per-round best action. The construction allows us to relate the learner’s loss to this frozen benchmark, while the error incurred by freezing time can be bounded by the amount of variation in the losses.

Define $m_t^\tau \in \arg\min_{i \in \Tilde{\mathcal{X}}_t} y_{\tau, i}$ to be the $\tau$-frozen best arm of time $t$. Denote the value associated with this arm $(y^*_t)^\tau = y_{\tau, m_t^\tau}$. Thus, we can define the time-frozen smooth comparator
\begin{equation} \label{eqn:time-frozen-defn}
\boldsymbol{q}_t^\tau = (1 - k\alpha)\boldsymbol{e}_{m_t^\tau} + \alpha\boldsymbol{1}_{\Tilde{\mathcal{X}}_t}
\end{equation}

\begin{lemma} \label{lemma:time-frozen-gap}
Let $\mathcal{I} = \{\tau, \dots, \tau + L-1\}$ be an arbitrary block. Suppose $\alpha > 0$. Let $V_n$ be as defined in Theorem~\ref{thm:regret-simplified}. For any sequence of random variables $\{X_t\}_{t \in \mathcal{I}}$ with a common support, we have
\begin{equation} \label{eqn:frozen-best-gap}
\EE \left[ \sum_{t \in \mathcal{I}} X_t - \left\langle \boldsymbol{q}_t, \frac{\boldsymbol{l}_t^{\Tilde{\mathcal{X}}_t}}{Z_t} \right\rangle \right] \leq \EE \left[ \sum_{t \in \mathcal{I}} X_t - \left\langle \boldsymbol{q}^\tau_t, \frac{\boldsymbol{l}_t^{\Tilde{\mathcal{X}}_t}}{Z_t} \right\rangle \right] + \frac{1 - k\alpha}{k\alpha}\sum_{t \in \mathcal{I}} \|\boldsymbol{l}_t - \boldsymbol{l}_\tau \|_1
\end{equation}
\end{lemma}

\begin{proof}
By the definition of $\boldsymbol{q}_t^\tau$, we have
\begin{equation} \label{eqn:lemma-time-frozen-gap-1}
\begin{aligned}
\left\langle \boldsymbol{q}_t^\tau, \boldsymbol{l}_\tau^{\Tilde{\mathcal{X}}_t} \right\rangle & \leq (1 - k\alpha)(y^*_t)^\tau + \alpha \sum_{j \in \Tilde{\mathcal{X}}_t} \boldsymbol{l}_{\tau, j} \\
& \leq (1 - k\alpha(y^*_t)^\tau + \alpha(k (y^*_t)^\tau + k - 1) \\
&= m_t^\tau + \alpha(k-1) \\
\end{aligned}
\end{equation}

Rearranging (\ref{eqn:lemma-time-frozen-gap-1}) and divide by $Z_t$, we get
\begin{equation} \label{eqn:relate-m-q}
\frac{m_t^\tau}{Z_t} \geq \left\langle \boldsymbol{q}_t^\tau, \frac{\boldsymbol{l}_\tau^{\Tilde{\mathcal{X}}_t}}{Z_t} \right\rangle - \frac{\alpha(k-1)}{Z_t} \geq \left\langle \boldsymbol{q}_t^\tau, \frac{\boldsymbol{l}_\tau^{\Tilde{\mathcal{X}}_t}}{Z_t} \right\rangle - \frac{k-1}{k}.
\end{equation}

Using the fact that $\min$ is a Lipschitz operation, we have
\[
\left| y^*_t - (y^*_t)^\tau \right| = \left| \min_{i\in \Tilde{\mathcal{X}}_t} \boldsymbol{l}_{t, i} - \min_{i\in \Tilde{\mathcal{X}}_t} \boldsymbol{l}_{\tau, i} \right| \leq \|\boldsymbol{l}_t - \boldsymbol{l}_\tau \|_\infty \leq  \|\boldsymbol{l}_t - \boldsymbol{l}_\tau \|_1,
\]

Applying this identity, we have
\begin{equation} \label{eqn:lemma-time-frozen-gap-2}
\begin{aligned}
\left| \left\langle \boldsymbol{q}_t, \frac{\boldsymbol{l}_t^{\Tilde{\mathcal{X}}_t}}{Z_t} \right\rangle - \left\langle \boldsymbol{q}^\tau_t, \frac{\boldsymbol{l}_t^{\Tilde{\mathcal{X}}_t}}{Z_t} \right\rangle \right| & \leq \frac{1}{k\alpha} \left| \left\langle \boldsymbol{q}_t, \boldsymbol{l}_t^{\Tilde{\mathcal{X}}_t} \right\rangle - \left\langle \boldsymbol{q}^\tau_t, \boldsymbol{l}_t^{\Tilde{\mathcal{X}}_t} \right\rangle \right| \\
& = \frac{1 - k\alpha}{k\alpha} \left| y^*_t - (y^*_t)^\tau \right| \\
& \leq \frac{1 - k\alpha}{k\alpha} \|\mathbf{y}_t - \mathbf{y}_\tau\|_1
\end{aligned}
\end{equation}

Rearrange (\ref{eqn:lemma-time-frozen-gap-2}) and add $X_t$ to both sides while taking expectation, we get
\begin{equation} \label{eqn:frozen-best-gap}
\EE \left[ \sum_{t \in \mathcal{I}} X_t - \left\langle \boldsymbol{q}_t, \frac{\boldsymbol{l}_t^{\Tilde{\mathcal{X}}_t}}{Z_t} \right\rangle \right] \leq \EE \left[ \sum_{t \in \mathcal{I}} X_t - \left\langle \mathbf{q}^\tau_t, \frac{\boldsymbol{l}_t^{\Tilde{\mathcal{X}}_t}}{Z_t} \right\rangle \right] + \frac{1 - k\alpha}{k\alpha}\sum_{t \in \mathcal{I}} \|\boldsymbol{l}_t - \boldsymbol{l}_\tau \|_1
\end{equation}
\end{proof}

\subsection{Moment Lemmas}
To apply the standard OMD bound (Theorem~\ref{thm:regret-simplified}), we require unbiased loss estimates with controlled variance. Since the learner only observes losses on the sampled subset ($\Tilde{\mathcal{X}}_t$), we work with an importance-weighted estimator that accounts for both subset sampling and the normalization induced by $Z_t$.

The following two lemmas establish the properties needed for the regret analysis. The first shows that the estimator is unbiased for the scaled loss, ensuring that the expected update direction matches the true loss. The second provides a bound on the second moment of the estimator, which controls the variance term in the OSMD regret bound. Together, these results justify the use of the estimator in the mirror descent analysis and quantify the cost introduced by partial observation and uniform exploration.

\begin{lemma}
[Unbiasedness]\label{lemma:unbiasedness}
Let $\widehat{L}_{t, i}$ be as defined in Algorithm \ref{alg:somd}. Denote $\widehat{\boldsymbol{L}}_t = (\widehat{L}_{t, 1}, \widehat{L}_{t, 2}, \dots, \widehat{L}_{t, K})$. For every $i \in [K]$,
\begin{equation}
\EE[\widehat{L}_{t, i} \mid \Tilde{\mathcal{X}}_t] = \boldsymbol{l}_{t, i} \mathbb{I}\{ i\in \Tilde{\mathcal{X}}_t \} = \boldsymbol{l}_t^{\Tilde{\mathcal{X}}_t}(i) \\
\end{equation}
Additionally, define the time-frozen estimator
\begin{equation}
\widehat{L}_{t, i}^\tau := \frac{1}{s} \sum_{r=1}^s \mathbb{I}\{a_{t, r} = i\} \frac{l_{\tau, i}}{\boldsymbol{p}^t(i \mid \Tilde{\mathcal{X}}_t)},
\end{equation}
and we have
\begin{equation}
\EE[\widehat{L}_{t, i}^\tau \mid \Tilde{\mathcal{X}}_t] = l_{\tau, i} \mathbb{I}\{i \in \Tilde{\mathcal{X}}_t\} = \boldsymbol{l}_\tau^{\Tilde{\mathcal{X}}_t}(i)
\end{equation}
\end{lemma}

\begin{proof}
If $i \in \Tilde{\mathcal{X}}_t$,
\[
\EE \left[ \widehat{L}_{t, i}^\tau \middle| \Tilde{\mathcal{X}}_t \right] = \frac{1}{s} \sum_{r=1}^s \boldsymbol{p}_t(i \mid \Tilde{\mathcal{X}}_t) \frac{l_{t, i} \mathbb{I}\{i \in \Tilde{\mathcal{X}}_t\}}{\boldsymbol{p}^t(i \mid \Tilde{\mathcal{X}}_t)} = l_{t, i} \mathbb{I}\{i \in \Tilde{\mathcal{X}}_t\}
\]
Similarly,
\[
\EE \left[ \widehat{L}_{t, i}^\tau \middle| \Tilde{\mathcal{X}}_t \right] = \frac{1}{s} \sum_{r=1}^s \boldsymbol{p}_t(i \mid \Tilde{\mathcal{X}}_t) \frac{l_{\tau, i} \mathbb{I}\{i \in \Tilde{\mathcal{X}}_t\}}{\boldsymbol{p}^t(i \mid \Tilde{\mathcal{X}}_t)} = l_{\tau, i} \mathbb{I}\{i \in \Tilde{\mathcal{X}}_t\}
\]
\end{proof}

\begin{lemma}[Bounded Second Moment] \label{lemma:second-moment-bound}
Let $\widehat{L}_{t, j}$ be as defined in Algorithm \ref{alg:somd}. Suppose we obtain $\boldsymbol{p}^t$ through
\begin{equation*}
\boldsymbol{p}^t \in \arg\min_{\boldsymbol{p} \in \mathcal{A}} \eta\left\langle \boldsymbol{p}, \widehat{\boldsymbol{L}}_t\right\rangle + D_F(\boldsymbol{p}, \boldsymbol{p}^{t-1}),
\end{equation*}
under the constraint $\boldsymbol{p}^t(i) \geq \alpha$ for all $i$. We have
\begin{equation}
\EE \left[ \sum_{j=1}^K \boldsymbol{p}^t_{j}\widehat{L}_{t, j}^2 \middle| \Tilde{\mathcal{X}}_t \right] \leq \frac{k}{s}
\end{equation}
\end{lemma}
\begin{proof}
For $j \not\in \Tilde{\mathcal{X}}_t$, $\widehat{L}_{t, j} = 0$, For $j \in \Tilde{\mathcal{X}}_t$, let
\[
X_{r, j} := \mathbb{I}\left\{ A_{t, r = j} \right\} \frac{\boldsymbol{l}_{\tau, j}}{\boldsymbol{p}^t(j \mid \Tilde{\mathcal{X}}_t)}, \quad \text{consequently } \widehat{L}_{t, j} = \frac{1}{s} \sum_{r=1}^s X_{r, j}
\]
By construction, $X_{r, j}$'s are IID in $r$, and since $\boldsymbol{l}_{\tau, j} \in [0, 1]$,
\[
\EE \left[ X^2_{r, j} \mid \Tilde{\mathcal{X}}_t \right] = 
\sum_{j \in \Tilde{\mathcal{X}}_t} \boldsymbol{p}^t(j \mid \Tilde{\mathcal{X}}_t) \frac{\boldsymbol{l}^2_{\tau, j}}{\boldsymbol{p}^t(j \mid \Tilde{\mathcal{X}}_t)^2} \leq \frac{1}{\boldsymbol{p}^t(j \mid \Tilde{\mathcal{X}}_t)}
\]
By Jensen's inequality, we have
\[
\EE\left[ (\widehat{L}_{t, j}^\tau)^2 \mid \Tilde{\mathcal{X}}_t \right] = \EE \left[ \left( \frac{1}{s} \sum_{r=1}^s X_{r, j} \right)^2 \middle| \Tilde{\mathcal{X}}_t \right] \leq \frac{1}{s} \EE\left[ X_{r, j}^2 \middle| \Tilde{\mathcal{X}}_t \right] \leq \frac{1}{s \boldsymbol{p}^t(j \mid \Tilde{\mathcal{X}}_t)}
\]
Summing with weights $\boldsymbol{p}^t$ over $j \in \Tilde{\mathcal{X}}_t$ gives
\[
\sum_{j\in \Tilde{\mathcal{X}}_t} \left( \boldsymbol{p}^t_{\mid \Tilde{\mathcal{X}}_t} \right)_j \EE\left[ (\widehat{L}_{t, j}^\tau)^2 \middle| \Tilde{\mathcal{X}}_t \right] \leq \frac{1}{s} \sum_{j \in \Tilde{\mathcal{X}}_t} \left( \boldsymbol{p}^t_{\mid \Tilde{\mathcal{X}}_t} \right)_j \frac{Z_t}{\left( \boldsymbol{p}^t_{\mid \Tilde{\mathcal{X}}_t} \right)_j} = \frac{k}{s}Z_t \leq \frac{k}{s}
\]
\end{proof}

\subsection{One-Step OMD Bound}

The purpose of the following lemma is to provide a local inequality that governs how a single mirror descent update behaves under the importance-weighted loss estimator. When summed over time, the KL terms telescope while the second-moment terms accumulate in a controlled manner. This structure allows the proof to separate the effect of the update rule from issues caused by partial observability and non-stationarity, which are handled in later steps.

\begin{lemma} \label{lemma:subset-omd-ineq}
Suppose $\boldsymbol{p}^t$'s are obtained as in Algorithm~\ref{alg:somd}. Then, for any comparator $\boldsymbol{v}_t \in \Delta_{\Tilde{\mathcal{X}}_t}$,
\begin{equation}
\left\langle \boldsymbol{p}^t_{\mid \Tilde{\mathcal{X}}_t} - \boldsymbol{v}_t, \boldsymbol{l}_t \right\rangle \leq \frac{1}{\eta} \left( \mathrm{KL}(\boldsymbol{v}_t \| \boldsymbol{p}^t_{\mid \Tilde{\mathcal{X}}_t} ) - \mathrm{KL}(\boldsymbol{v}_t \| \boldsymbol{p}^{t+1}_{\mid \Tilde{\mathcal{X}}_{t+1}}) \right) + \frac{\eta s^2}{2} \sum_{i \in C_t} \boldsymbol{p}^t(i \mid \Tilde{\mathcal{X}}_t) \widehat{L}_{t, i}^2
\end{equation}
\end{lemma}

\begin{proof}
For the simplicity of the proof, fix $t$. Write $C = \Tilde{\mathcal{X}}_t$. Let $p_i = \boldsymbol{p}^t_i$, $Z = Z_t$, $q_i = \boldsymbol{p}^t(i \mid \Tilde{\mathcal{X}}_t)$, and $\widehat{L}_i = \widehat{L}_{t, i}$. Write the updated weights as $p_i^+ = p_i \exp(-\eta\widehat{L}_i)$ and the updated active sum as $Z^+ = \sum_{i \in C} p^+$. Let $q_i^+ = p_i^+ / Z^+$. We denote the vector form of the aforementioned quantities in bold. e.g. $\boldsymbol{p}(i) = p_i$.

Note
\begin{equation}
\log \frac{q_i}{q_i^+} = \log \frac{p_i / Z}{p_i \exp(-\eta \widehat{L}_i) / Z^+} = \log \frac{Z^+}{Z} + \eta \widehat{L}_i
\end{equation}

Now,
\begin{equation}
\begin{aligned}
\mathrm{KL}(\boldsymbol{v} \| \boldsymbol{q}^+) - \mathrm{KL}(\boldsymbol{v} \| \boldsymbol{q}) & = \sum_{i \in C} v_i \log \frac{v_i}{q^+_i} - \sum_{i \in C} v_i \log \frac{v_i}{q_i} \\
& = \sum_{i \in C} v_i \log \frac{q_i}{q^+_i} \\
& = \log \frac{Z^+}{Z} + \eta \left\langle \boldsymbol{v}, \widehat{\boldsymbol{L}} \right\rangle
\end{aligned}
\end{equation}
Rearranging we have
\begin{equation} \label{eqn:subset-ineq-1}
\eta \left\langle \boldsymbol{v}, \widehat{\boldsymbol{L}} \right\rangle = \mathrm{KL}(\boldsymbol{v} \| \boldsymbol{q}^+) - \mathrm{KL}(\boldsymbol{v} \| \boldsymbol{q}) - \log \frac{Z^+}{Z}
\end{equation}

Next, apply the fact that $e^{-x} \leq 1 - x + x^2$ for $x \geq 0$ and taking logs, we have
\begin{equation} \label{eqn:subset-ineq-2}
\begin{aligned}
\log \frac{Z^+}{Z} & \leq \log\left( \sum_{i \in C} q_i \left( 1 - \eta \widehat{L}_i + \frac{\eta^2}{2}\widehat{L}_i^2 \right) \right) \\
& = \log\left( 1 - \eta \left\langle \boldsymbol{q}, \widehat{\boldsymbol{L}} \right\rangle + \frac{\eta^2}{2} \sum_{i \in C} q_i \widehat{L}_i^2\right) \\
& \leq -\eta \left\langle\boldsymbol{q}, \widehat{\boldsymbol{L}} \right\rangle + \frac{\eta^2}{2} \sum_{i \in C} q_i \widehat{L}_i^2
\end{aligned}
\end{equation}
where the last inequality applies $\log(1 + x) \leq x$ for all $x > -1$.

Now plug (\ref{eqn:subset-ineq-2}) into (\ref{eqn:subset-ineq-1}),
\begin{equation}
\eta \left\langle \boldsymbol{v}, \hat{\boldsymbol{L}} \right\rangle \geq \mathrm{KL}(\boldsymbol{v} \| \boldsymbol{q}^+) - \mathrm{KL}(\boldsymbol{v} \| \boldsymbol{q}) + \eta \left\langle\boldsymbol{q}, \hat{\boldsymbol{y}} \right\rangle + \frac{\eta^2}{2} \sum_{i \in C} q_i \widehat{L}_i^2
\end{equation}
Rearranging the inequality and dividing by $\eta$, we get
\begin{equation}
\langle \boldsymbol{q} - \boldsymbol{v}, \mathbf{\widehat{L}} \rangle \leq \frac{1}{\eta}(\mathrm{KL}(\boldsymbol{v} \| \boldsymbol{q}) - \mathrm{KL}(\boldsymbol{v} \| \boldsymbol{q})) + \frac{\eta}{2} \sum_{i \in C} q_i \widehat{L}_i^2
\end{equation}
Substituting the original notation back, we recover the claim.
\end{proof}

\subsection{Reduction to Blockwise Fixed Arm}

The next two lemmas connect the one-step bound to a blockwise analysis under non-stationarity. Lemma~\ref{lemma:block-per-turn-gap} controls the discrepancy between the per-round best available arm and a single arm fixed over a block, showing that this gap is governed by the cumulative variation within the block. Lemma~\ref{lemma:unopt-regret} then combines this control with the one-step OMD bound to obtain a regret bound against a fixed comparator over the block. Together, these results allow the per-round inequalities to be aggregated while isolating the effect of non-stationarity.

\begin{lemma} \label{lemma:block-per-turn-gap}
Fix a block $\mathcal{B} = \{t_0, t_0 + 1, \dots, t_0 + L - 1\}$. Dentoe the blockwise best arm as
\begin{equation}
m_\mathcal{B} \in \arg\min_{i \in [K]} \sum_{i \in \mathcal{B}} \widehat{L}_{t, i} \mathbb{I}\{i \in \Tilde{\mathcal{X}}_t \}
\end{equation}
Let $\Delta_t$ be as defined in Theorem~\ref{thm:regret-simplified}. Then,
\begin{equation}
\sum_{t \in \mathcal{B}} \left(\widehat{L}_{t, m_\mathcal{B}} \mathbb{I}\{m_\mathcal{B} \in C_t\} - \widehat{L}_{t, m_t} \right) \leq L \sum_{t \in B \setminus \{t_0\}} \Delta_t
\end{equation}
\end{lemma}

\begin{proof}
Let $b_t(C) = \min_{i\in C} \widehat{L}_{i, t}$. By the Lipschitz property of $\min$, it satisfies
\begin{equation} \label{eqn:block-variation-bound-1}
\left| b_t(C) - b_{t-1}(C) \right| \leq \max_{i \in C} \left| \widehat{L}_{t, i} - \widehat{L}_{t-1, i} \right| \leq \Delta_t
\end{equation}
By telescoping (\ref{eqn:block-variation-bound-1}) from $t_0$ to $t$,
\begin{equation}
b_{t_0}(C_t) \leq b_t(C_t) + \sum_{\tau=t_0+1}^t \Delta_\tau = m_t + \sum_{\tau=t_0+1}^t \Delta_\tau
\end{equation}
By definition of $m_\mathcal{B}$ as the best static arm for the block
\begin{equation}
\sum_{t \in \mathcal{B}} \widehat{L}_{t, m_\mathcal{B}} \mathbb{I}\{b_\mathcal{B} \in \Tilde{\mathcal{X}}_t\} \leq \sum_{t \in \mathcal{B}} \widehat{L}_{t, m_t^{t_0}} \mathbb{I}\{m_t^{t_0} \in \Tilde{\mathcal{X}}_t\}
\end{equation}
Also note that by definition,
\begin{equation} \label{eqn:block-variation-bound-2}
\widehat{L}_{t, m_t^{t_0}} \leq \widehat{L}_{t_0, m_t^{t_0}} + \sum_{\tau=t_0 + 1}^t \Delta_\tau = b_{t_0}(\Tilde{\mathcal{X}}_t) + \sum_{\tau=t_0+1}^t \Delta_\tau
\end{equation}
Combine (\ref{eqn:block-variation-bound-1}) and (\ref{eqn:block-variation-bound-2}),
\begin{equation}
\sum_{t \in \mathcal{B}} \widehat{L}_{t, m_\mathcal{B}} \mathbb{I}\{m_\mathcal{B} \in \Tilde{\mathcal{X}}_t\} \leq \sum_{t \in \mathcal{B}} \left(\widehat{L}_{t, m_t} + \sum_{\tau=t_0+1}^t \Delta_\tau \right)
\end{equation}
Rearranging, we get
\begin{equation}
\sum_{t \in \mathcal{B}} \widehat{L}_{t, m_\mathcal{B}} \mathbb{I}\{m_\mathcal{B} \in \Tilde{\mathcal{X}}_t\} - \widehat{L}_{t, m_t} \leq L \sum_{t \in \mathcal{B} \setminus \{t_0\}} \Delta_t
\end{equation}
\end{proof}

\begin{lemma} \label{lemma:unopt-regret}
Fix a block $\mathcal{B} = \{t_0, t_0 + 1, \dots, t_0 + L - 1\}$. Run Algorithm \ref{alg:somd} with $\eta > 0$ starting from time $t_0$. Let $\boldsymbol{q}_t \in \Delta_{\Tilde{\mathcal{X}}_t}$ be the played distribution at each round $t \in \mathcal{B}$ and let $\widehat{\boldsymbol{L}}_t$ be an unbiased estimator of the masked loss vector
\begin{equation*}
\EE[\widehat{\boldsymbol{L}}_t \mid \Tilde{\mathcal{X}}_t] = \boldsymbol{L}_t^{\Tilde{\mathcal{X}}_t}
\end{equation*}
and that the second moment is bounded by
\begin{equation}
\EE\left[ \sum_{i \in \Tilde{\mathcal{X}}_t} \boldsymbol{q}_{t, i} \widehat{L}_{t, i}^2 \right] \leq \frac{k}{s}
\end{equation}
for all $t$. Then,
\begin{equation} \label{eqn:unopt-regret-statement}
\EE \left[ \sum_{t \in \mathcal{B}} \left\langle \boldsymbol{q}_t, \boldsymbol{l}_t^{C_t} \right\rangle - \boldsymbol{l}_{t, m_\mathcal{B}} \mathbb{I}\{m_\mathcal{B} \in \Tilde{\mathcal{X}}_t\} \right] \leq \frac{\log(1/\alpha)}{\eta} + \frac{s k \eta L}{2}
\end{equation}
\end{lemma}

\begin{proof}
For each $t \in \mathcal{B}$ and any comparators $\boldsymbol{v}_t \in \Delta_{\Tilde{\mathcal{X}}_t}$, apply Lemma \ref{lemma:subset-omd-ineq}
\begin{equation} \label{eqn:unopt-regret-1}
\sum_{t \in \mathcal{B}} \left\langle \boldsymbol{q}_t - \boldsymbol{v}_t, \boldsymbol{L}_t \right\rangle \leq \frac{1}{\eta} \sum_{t \in \mathcal{B}} \left( \mathrm{KL}(\boldsymbol{v}_t \| \boldsymbol{q}_t) - \mathrm{KL}(\boldsymbol{v}_t \| \boldsymbol{q}_{t+1}) \right) + \frac{\eta s^2}{2} \sum_{t \in \mathcal{B}} \sum_{i \in \Tilde{\mathcal{X}}_t} \boldsymbol{q}_{t, i} \widehat{L}_{t, i}^2
\end{equation}
By the unbiasedness of $\widehat{\mathbf{y}}_t$,
\begin{equation*}
\EE[ \left\langle \boldsymbol{q}_t - \boldsymbol{v}_t, \boldsymbol{l}_t \right\rangle ] = \EE\left[ \left\langle \boldsymbol{q}_t - \boldsymbol{v}_t, \boldsymbol{l}_t^{\Tilde{\mathcal{X}}_t} \right\rangle \right]
\end{equation*}
For the KL terms, we have
\begin{equation} \label{eqn:unopt-regret-2}
\begin{aligned}
\sum_{t \in \mathcal{B}} \left( \mathrm{KL}(\boldsymbol{v}_t \| \boldsymbol{q}_t) - \mathrm{KL}(\boldsymbol{v}_t \| \boldsymbol{q}_{t+1}) \right) & = \mathrm{KL}(\boldsymbol{v}_t \| \boldsymbol{q}_{t_0}) - \mathrm{KL}(\boldsymbol{v}_t \| \boldsymbol{q}_{t_0 + L - 1}) \\
& \leq \log(1/\alpha)
\end{aligned}
\end{equation}
where the first equality is a result of telescoping with respect to $t$ and the second due to the fact that $\mathbf{q}_{t_0}$ is restarted from uniform. 

For the second term on the RHS of (\ref{eqn:unopt-regret-1}), we invoke Lemma \ref{lemma:second-moment-bound}. Taking expectation on both sides, we get (\ref{eqn:unopt-regret-statement}).
\end{proof}

\subsection{Proof of Theorem~\ref{thm:regret-simplified}}
Note that
\begin{equation}
\sum_{t \in \mathcal{B}} \left\langle \boldsymbol{q}_t, \boldsymbol{l}_t^{\Tilde{\mathcal{X}}_t} \right\rangle - \boldsymbol{l}_{t, m_t} = \left( \sum_{t \in \mathcal{B}} \left\langle \boldsymbol{q}_t, \boldsymbol{l}_t^{\Tilde{\mathcal{X}}_t} \right\rangle - \boldsymbol{l}_{t, m_\mathcal{B}} \mathbb{I}\{ m_\mathcal{B} \in \Tilde{\mathcal{X}}_t \} \right) + \left( \sum_{t \in \mathcal{B}} \boldsymbol{l}_{t, m_\mathcal{B}} \mathbb{I}\{ m_\mathcal{B} \in \Tilde{\mathcal{X}}_t \} - \boldsymbol{l}_{t, m_t} \right)
\end{equation}
We bound the first term using Lemma \ref{lemma:unopt-regret} and the second term using \ref{lemma:block-per-turn-gap}. Taking the expectation and summing over blocks, we get
\begin{equation}
\mathrm{Reg}_n^{\mathrm{BA}} = \EE\left[ \sum_{i=1}^B \sum_{t \in \mathcal{B}_i} \left\langle \boldsymbol{q}_t, \boldsymbol{l}_t^{\Tilde{\mathcal{X}}_t} \right\rangle - \boldsymbol{l}_{t, m_t} \right] \leq \frac{T}{L}\frac{\log(1/\alpha)}{\eta} + \frac{\eta sk}{2}T + LV_n
\end{equation}

Take 
\begin{align}
L^* 
&= \left( \frac{T}{V_n}\sqrt{\frac{s k \log(1/\alpha)}{2}} \right)^{2/3}, \\
\eta^*
&= 2^{2/3}\,(\log(1/\alpha))^{1/3}\,(s k)^{-2/3}
\left(\frac{V_n}{T}\right)^{1/3}.
\end{align}
we get
\begin{equation}
\mathrm{Reg}_n^{\mathrm{BA}}
\;\le\;
\frac{3}{2^{1/3}}\,
\big(T^{2} \, s k \, \log(1/\alpha)\, V_n \big)^{1/3}.
\end{equation}
\newpage

\section{Performance improvement contribution of each problem}
\label{sec:utility_contribution}

We provide further justification for the per-problem policy improvement contribution
defined in \cref{eq:per_problem_utility}. Intuitively, this quantity measures the marginal
contribution of including a training problem $\boldsymbol{x}$ in an actor update to the
overall improvement in policy performance under the evaluation distribution
$p_{\mathcal{X}}$.

This interpretation is exact for a broad class of \emph{tabular policies}, where the policy
parameters for different problems (or states) are independent. In such settings, updating
the policy using trajectories from a problem $\boldsymbol{x}$ affects only the conditional
distribution $\pi(\cdot \mid \boldsymbol{x})$ and leaves the policy unchanged on all other
problems.

\paragraph{Example: tabular REINFORCE.}
Consider tabular REINFORCE, where for each problem $\boldsymbol{x}$ the policy
$\pi(\cdot \mid \boldsymbol{x})$ is parameterized independently.
At iteration $t$, suppose we update the policy using rollouts collected \emph{only} from a
single problem $\boldsymbol{x}$.
By construction, this update modifies $\pi(\cdot \mid \boldsymbol{x})$ but does not change
$\pi(\cdot \mid \boldsymbol{x}')$ for any $\boldsymbol{x}' \neq \boldsymbol{x}$.

The overall performance objective is
\[
J(\pi)
=
\sum_{\boldsymbol{x}' \in \mathcal{X}}
p_{\mathcal{X}}(\boldsymbol{x}')
\,
\mathbb{E}_{\boldsymbol{y} \sim \pi(\cdot \mid \boldsymbol{x}')}
\bigl[
R(\boldsymbol{y} \mid \boldsymbol{x}')
\bigr].
\]
Since only the conditional policy at $\boldsymbol{x}$ is changed, the performance difference
$J(\pi^{t+1}) - J(\pi^t)$ depends solely on how the expected reward at $\boldsymbol{x}$
changes.

Applying the standard performance difference identity
\citep{kakade2002approximately} to the single-turn setting yields
\[
J(\pi^{t+1}) - J(\pi^t)
=
p_{\mathcal{X}}(\boldsymbol{x})
\,
\mathbb{E}_{\boldsymbol{y} \sim \pi^{t}(\cdot \mid \boldsymbol{x})}
\left[
\frac{\pi^{t+1}(\boldsymbol{y} \mid \boldsymbol{x})}
{\pi^{t}(\boldsymbol{y} \mid \boldsymbol{x})}
A_{\pi^t}(\boldsymbol{y} \mid \boldsymbol{x})
\right],
\]
which is exactly the per-problem utility $u^t_{\boldsymbol{x}}$ defined in
\cref{eq:per_problem_utility}.
Thus, in tabular REINFORCE, updating the policy on a single problem $\boldsymbol{x}$
produces an expected performance improvement of $u^t_{\boldsymbol{x}}$.

The same reasoning extends directly to batch updates: if the policy is updated using a
subset of problems $\mathcal{X}^t$, the total performance improvement decomposes additively
as $\sum_{\boldsymbol{x} \in \mathcal{X}^t} u^t_{\boldsymbol{x}}$, and including a problem
$\boldsymbol{x}$ in the update contributes exactly $u^t_{\boldsymbol{x}}$ to the expected
performance gain.

\paragraph{Other tabular methods.}
This exact additive interpretation applies equally to other tabular reinforcement learning
methods, including tabular policy gradient methods, tabular Q-learning, and SARSA, where
updates based on a problem or state affect only the corresponding local policy or value
parameters. In all such cases, the per-problem utility $u^t_{\boldsymbol{x}}$ captures the
true marginal contribution of training on $\boldsymbol{x}$.

\paragraph{Discussion: function approximation.}
In the presence of function approximation, updating the policy using trajectories from one
problem generally affects the policy on other problems as well, breaking the exact
additivity described above. Analyzing such cross-problem interference requires strong
assumptions on the structure of the function class and the optimization dynamics, and is
beyond the scope of this work.

Nevertheless, when policy updates are small—as is typical in modern RL post-training
algorithms such as PPO-style methods, GRPO, or GSPO—the per-problem utility
$u^t_{\boldsymbol{x}}$ remains a first-order approximation to the marginal contribution of
problem $\boldsymbol{x}$ to performance improvement. Our empirical results in
\cref{sec:main_results} indicate that this approximation is sufficiently accurate to drive
effective curriculum learning at scale.

\newpage

\newpage






\section{Experimental setup}
\label{sec:experimental_setup}

\subsection{Hyper-parameters}
\label{sec:hyperparameters}

Unless otherwise specified, all experiments use the same hyper-parameter configuration for \method. The complete set of hyper-parameters is summarized in Table~\ref{tab:hyperparams}. We briefly explain key parameters below, with a focus on those specific to the actor--curator framework.

\begin{table}[h]
\centering
\small
\begin{tabular}{ll}
\toprule
\textbf{Parameter} & \textbf{Value} \\
\midrule
\multicolumn{2}{l}{\textit{Model configuration}} \\
Curator model & Qwen3-0.6B \\
Actor KL loss & Disabled \\
Max problem length & 1024 tokens \\
Max solution length & 4096 tokens \\
\midrule
\multicolumn{2}{l}{\textit{Sampling and batch sizes}} \\
Candidate batch size $|\tilde{\mathcal{X}}^t|$ & 2048 \\
Training batch size $|\mathcal{X}^t|$ & 256 \\
Rollouts per problem $|\mathcal{Y}^t_{\boldsymbol{x}}|$ & 8 \\
\midrule
\multicolumn{2}{l}{\textit{Actor optimization}} \\
Actor temperature & 1.0 \\
Actor training top-$p$ & 1.0 \\
Actor validation top-$p$ & 0.7 \\
Actor learning rate & $1\times10^{-6}$ \\
Actor LR warmup ratio & 0.05 \\
Actor weight decay & 0.1 \\
Actor gradient clipping & 1.0 \\
Actor clip range $\rho_{\min}, \rho_{\max}$ & $[3\times10^{-4},\,4\times10^{-4}]$ \\
\midrule
\multicolumn{2}{l}{\textit{Curator optimization}} \\
Curator dormant steps & 20 \\
Curator warm-up steps & 5 \\
Curator temperature & 1.0 \\
Curator top-$p$ & 0.9 \\
Curator learning rate & $1\times10^{-6}$ \\
Curator PPO clip range $\rho_{\max}-\rho_{\min}$ & 0.2 \\
\bottomrule
\end{tabular}
\caption{Hyper-parameters used for \method across all experiments unless otherwise specified.}
\label{tab:hyperparams}
\end{table}

\paragraph{Candidate and training batch sizes.}
At each training iteration, a candidate batch $\tilde{\mathcal{X}}^t$ of size 2048 is first sampled from the proposal distribution $q$. The curator then reweights this candidate set and samples a smaller training batch $\mathcal{X}^t$ of size 256, which is used for actor rollouts and updates. This two-stage sampling scheme follows Section~3.5 and allows scalable curation over large problem banks.

\paragraph{Rollouts per problem.}
For each selected problem $x \in \mathcal{X}^t$, we generate $|\mathcal{Y}^t_x|=8$ on-policy rollouts from the current actor. These rollouts are used both for the actor update and for estimating per-problem policy improvement signals used to train the curator.

\paragraph{Curator dormant and warm-up steps.}
During the first \emph{curator dormant steps} (20 iterations), problems are sampled uniformly from the candidate batch, and curator outputs are ignored. This stabilizes early actor learning before meaningful policy-improvement estimates can be obtained. During the subsequent \emph{curator warm-up steps} (5 iterations), the curator begins to influence sampling, but its parameters are updated conservatively. After warm-up, the curator is fully active and trained online using bandit feedback.

\paragraph{Sampling prior.}
When \texttt{use sampling prior} is enabled, curator-assigned weights are multiplied by the proposal-induced marginal inclusion probability $q(x)$ before normalization. This encourages coverage of the full dataset and prevents the curator from collapsing onto a narrow subset of problems early in training, consistent with the two-stage unbiased estimator in Equation~(12).

\paragraph{Proximal curator clipping.}
Curator updates use the PPO-style clipped OSMD objective described in Section~3.6. The clipping range $\rho_{\min}, \rho_{\max}$ constrains the importance ratio between consecutive curator policies, stabilizing learning under function approximation. The additional clip range parameter controls the maximum allowed deviation between these bounds.

\subsection{Hardware}
\label{sec:hardware}

All experiments were conducted on NVIDIA A100 and H200 GPUs. Actor rollout generation and optimization dominate overall runtime; curator training introduces approximately 14\% additional wall-clock cost relative to uniform sampling, as discussed in Appendix~E.

\newpage
\section{Additional analysis}
\label{sec:analysis_plus}

\subsection{Overhead}
We found that \method adds about 9\% training wall time overhead. We present the overhead by dataset and model in \cref{tab:overhead}.
However, as shown in \cref{fig:efficiency-all}, this is relatively the compared to the efficiency gains. 

\begin{table}[h]
  \caption{\textbf{Average wall-time overhead percentage} by model and datasets, averaged over training steps 500 steps.}
  \label{tab:overhead}
  \vskip 0.15in
  \centering
    \begin{small}
      \begin{sc}
        \setlength{\tabcolsep}{4pt}
        \renewcommand{\arraystretch}{1.05}
        \begin{tabular}{ccccc}
          \toprule
           & \textbf{Countdown} & \textbf{Zebra} & \textbf{ARC-1D} & \textbf{MATH} \\
          \midrule
          Qwen2.5-3B-Base & 11.86 & 16.71 & 17.12 & 9.56\\
          \cmidrule(lr){2-5}
          Llama3.2-3B-IT & 9.34 & 13.63 & 19.40 & 9.82\\
          \bottomrule
        \end{tabular}
      \end{sc}
    \end{small}
  \vskip -0.1in
\end{table}

\newpage
\section{Additional Ablation} \label{appendix:ablation}

\paragraph{Candidate batch size $|\tilde{\mathcal{X}}|$.}
As shown in \cref{fig:ab_candidate_batch_size}, candidate sizes of $512$ and $2048$ yield similar final performance.
Larger batches (e.g., $8192$) lead to unstable training, likely due to reduced exploration and overfitting.

\begin{figure}[h]
    \centering
    \includegraphics[width=0.5\linewidth]{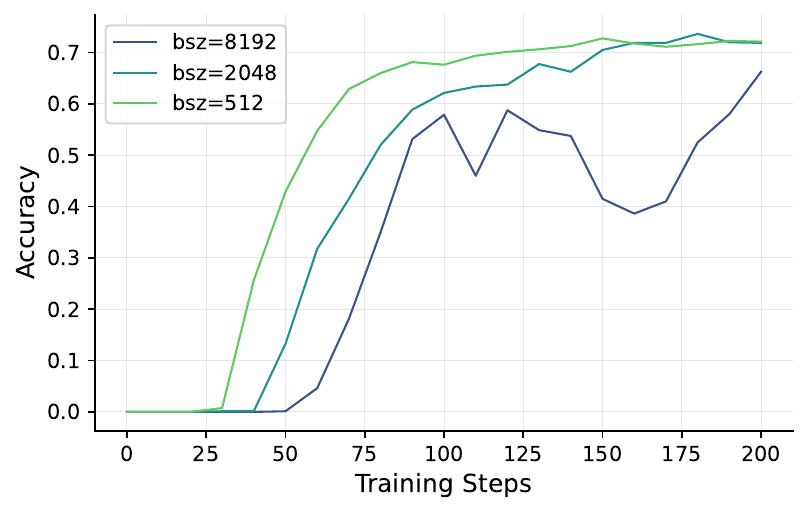}
    \caption{\textbf{Effect of candidate batch size on test performance}. While \texttt{bsz=2048} and \texttt{bsz=512} converge to similar performance at step 200, \texttt{bsz=8192} suffers from instability. Selected batch size is held constant at 256.}
    \label{fig:ab_candidate_batch_size}
\end{figure}
\newpage
\section{Main results (extended)}
\label{sec:main_results_plus}

This appendix provides extended empirical results complementing the main paper. We report additional quantitative comparisons across benchmarks, models, and curriculum learning methods, as well as detailed training dynamics over time. These results further substantiate the robustness, efficiency, and stability of \meth across diverse problem domains and model backbones.

\paragraph{Extended performance comparison.}
\Cref{tab:performance_full} reports peak validation performance within the first 100 training steps across all benchmarks and models. Compared to uniform sampling, heuristic curricula (SEC), and learning-based baselines (PCL), \meth consistently achieves higher peak performance on most benchmarks. The gains are particularly pronounced on harder subsets (e.g., Countdown-hard, Zebra-hard, ARC-hard, and AIME24), highlighting the effectiveness of directly optimizing for expected policy improvement when problem difficulty and utility are highly non-uniform. While performance on MATH500 is largely saturated for some model configurations, \meth remains competitive and avoids degradation relative to strong baselines.

\paragraph{Training dynamics and stability.}
\Cref{fig:testcurves-qwen25-3b,fig:testcurves-llama32-3b} visualize test performance as a function of training steps for Qwen2.5-3B-Base and Llama3.2-3B-it, respectively. Across benchmarks, \meth not only reaches higher peak performance but also exhibits faster convergence and more stable learning dynamics. In many cases, competing methods plateau early or exhibit higher variance, whereas \meth continues to make steady progress, effectively raising the performance ceiling.

\paragraph{Model-agnostic behavior.}
The trends observed in \cref{fig:testcurves-qwen25-3b,fig:testcurves-llama32-3b} are consistent across both base and instruction-tuned models, indicating that the benefits of \meth are not tied to a specific initialization or training regime. This supports the claim that learning curricula via policy-improvement-driven signals provides a generally applicable mechanism for improving RL post-training efficiency and robustness.

\paragraph{Training efficiency.}
\Cref{fig:efficiency-all} compares learning curves of Actor-Curator against uniform sampling on Countdown, Zebra, and ARC. Across all three benchmarks, Actor-Curator reaches the same target accuracy substantially earlier, yielding step-level speedups of $58.2\%$ on Countdown, $80.7\%$ on Zebra, and $24.3\%$ on ARC. This indicates that policy-improvement-driven data selection primarily accelerates optimization by prioritizing high-impact problems.


\begin{figure*}[t]
  \centering
  \caption{Test set performance across training steps (within 100 training steps) for \textbf{Qwen2.5-3B-Base} across benchmarks and methods. 
  }
  \label{fig:testcurves-qwen25-3b}
  \vspace{0.5em}

  \begin{subfigure}[t]{0.4\textwidth}
    \centering
    \includegraphics[width=\linewidth]{sections/figs/countdown_acc_qwen.pdf}
    \caption{\textbf{Countdown}}
    \label{fig:qwen-countdown}
  \end{subfigure}\hfill
  \begin{subfigure}[t]{0.4\textwidth}
    \centering
    \includegraphics[width=\linewidth]{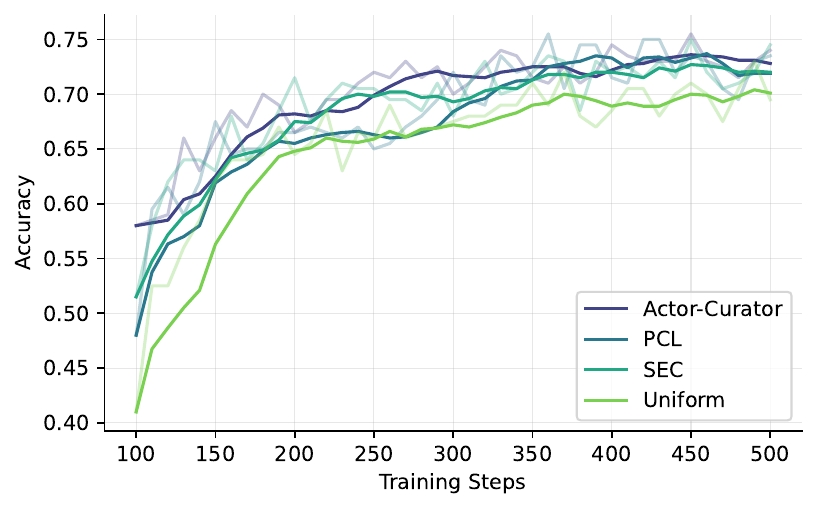}
    \caption{\textbf{Countdown-hard}}
    \label{fig:qwen-countdown-hard}
  \end{subfigure}

  \begin{subfigure}[t]{0.4\textwidth}
    \centering
    \includegraphics[width=\linewidth]{sections/figs/zebra_acc_qwen.pdf}
    \caption{\textbf{Zebra}}
    \label{fig:qwen-zebra}
  \end{subfigure}\hfill
  \begin{subfigure}[t]{0.4\textwidth}
    \centering
    \includegraphics[width=\linewidth]{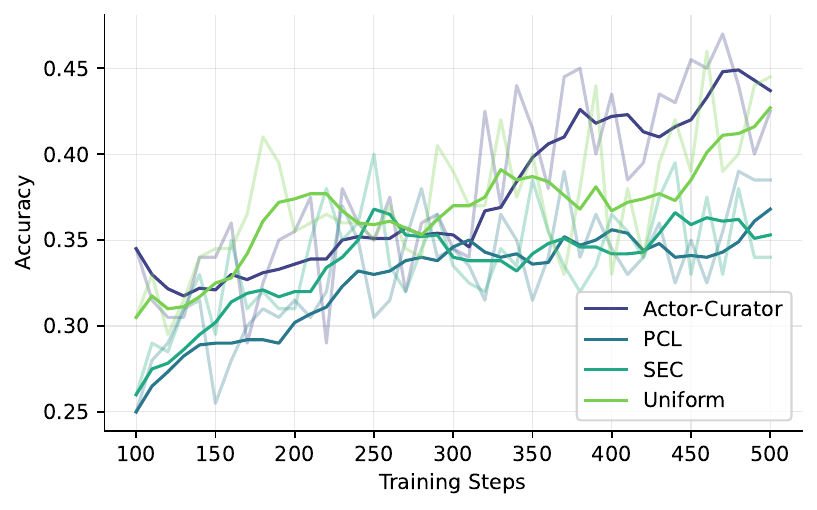}
    \caption{\textbf{Zebra-hard}}
    \label{fig:qwen-zebra-hard}
  \end{subfigure}

  \begin{subfigure}[t]{0.4\textwidth}
    \centering
    \includegraphics[width=\linewidth]{sections/figs/arc_acc_qwen.pdf}
    \caption{\textbf{ARC-1D}}
    \label{fig:qwen-arc-1d}
  \end{subfigure}\hfill
  \begin{subfigure}[t]{0.4\textwidth}
    \centering
    \includegraphics[width=\linewidth]{sections/figs/arc_acc_qwen.pdf}
    \caption{\textbf{ARC-hard}}
    \label{fig:qwen-arc-hard}
  \end{subfigure}

  \begin{subfigure}[t]{0.40\textwidth}
    \centering
    \includegraphics[width=\linewidth]{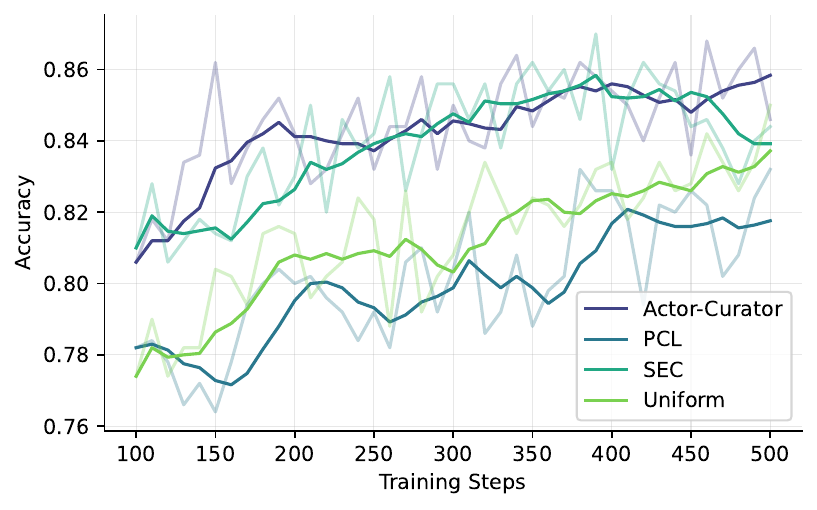}
    \caption{\textbf{MATH500}}
    \label{fig:qwen-math500}
  \end{subfigure}\hfill
  \begin{subfigure}[t]{0.40\textwidth}
    \centering
    \includegraphics[width=\linewidth]{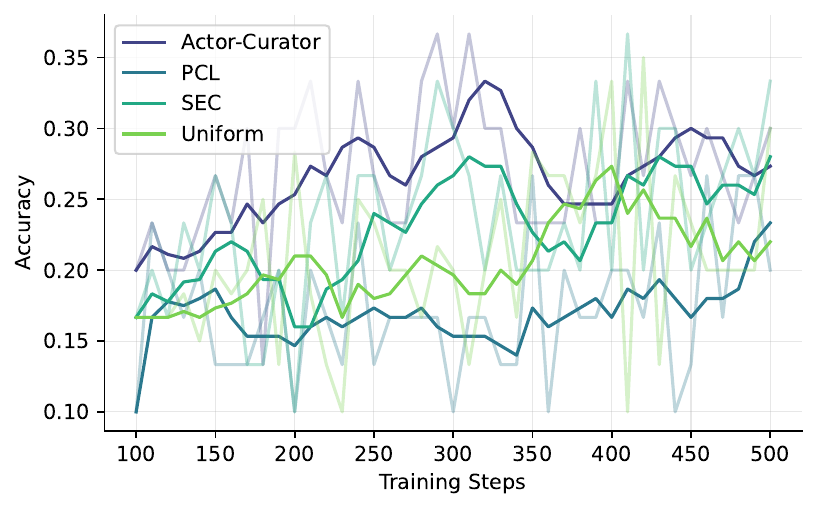}
    \caption{\textbf{AIME24}}
    \label{fig:qwen-aime24}
  \end{subfigure}

\end{figure*}

\begin{figure*}[t]
  \centering
  \caption{Test set performance across training steps (within 100 training steps) for \textbf{Llama3.2-3B-it} across benchmarks and methods. 
  }
  \label{fig:testcurves-llama32-3b}
  \vspace{0.5em}

  \begin{subfigure}[t]{0.40\textwidth}
    \centering
    \includegraphics[width=\linewidth]{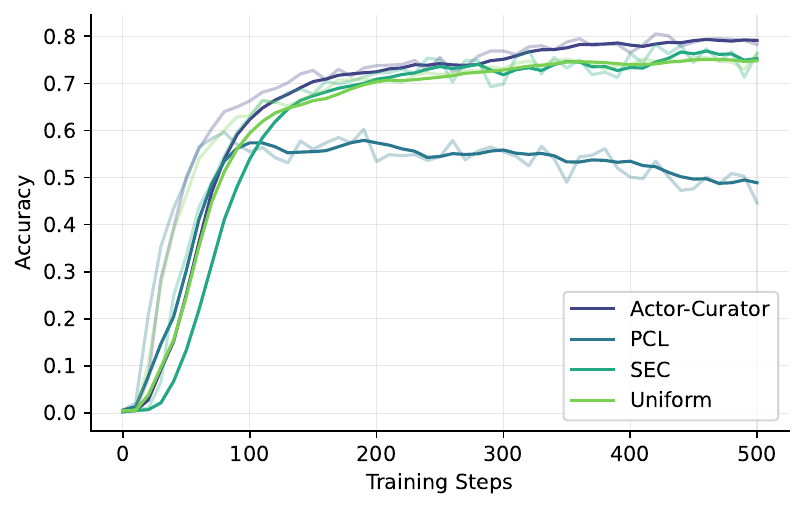}
    \caption{\textbf{Countdown}}
    \label{fig:llama-countdown}
  \end{subfigure}\hfill
  \begin{subfigure}[t]{0.40\textwidth}
    \centering
    \includegraphics[width=\linewidth]{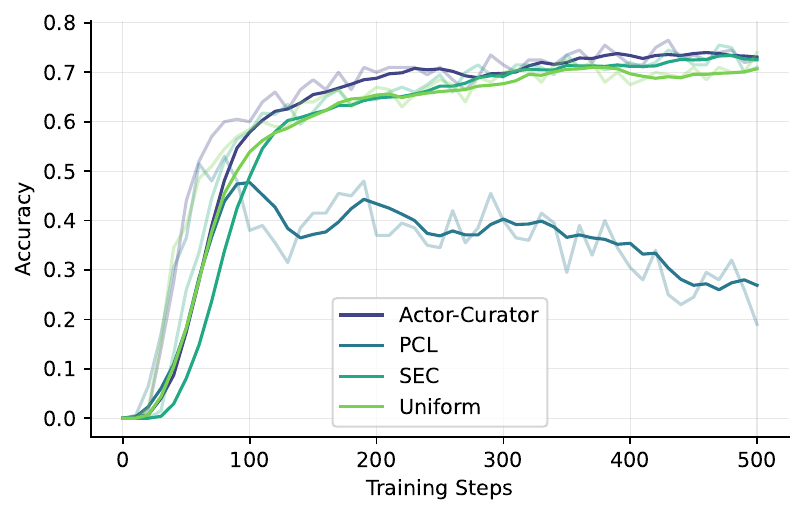}
    \caption{\textbf{Countdown-hard}}
    \label{fig:llama-countdown-hard}
  \end{subfigure}

  \vspace{0.7em}

  \begin{subfigure}[t]{0.40\textwidth}
    \centering
    \includegraphics[width=\linewidth]{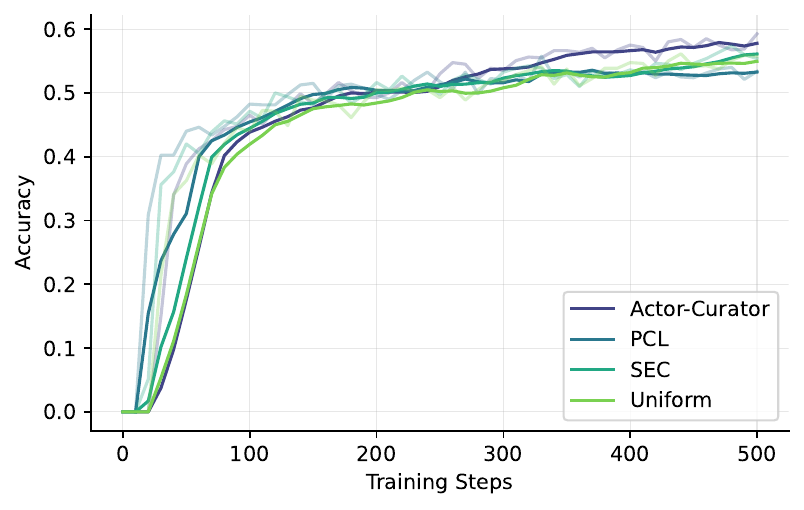}
    \caption{\textbf{Zebra}}
    \label{fig:llama-zebra}
  \end{subfigure}\hfill
  \begin{subfigure}[t]{0.40\textwidth}
    \centering
    \includegraphics[width=\linewidth]{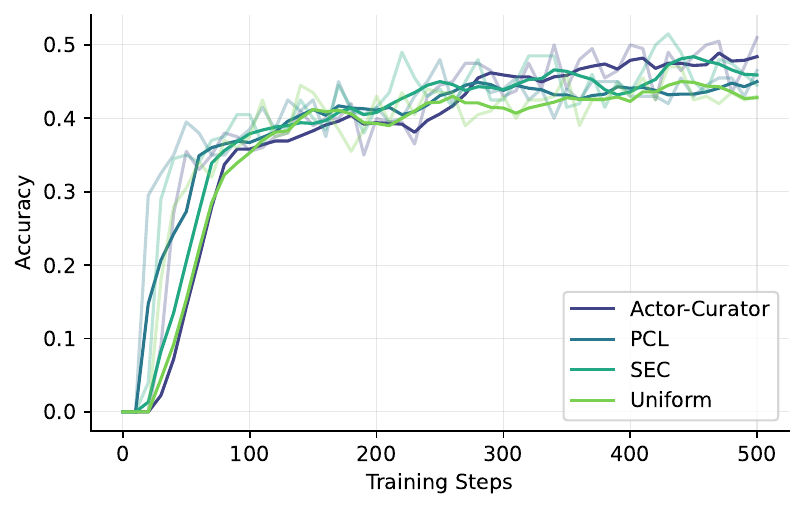}
    \caption{\textbf{Zebra-hard}}
    \label{fig:llama-zebra-hard}
  \end{subfigure}

  \vspace{0.7em}

  \begin{subfigure}[t]{0.40\textwidth}
    \centering
    \includegraphics[width=\linewidth]{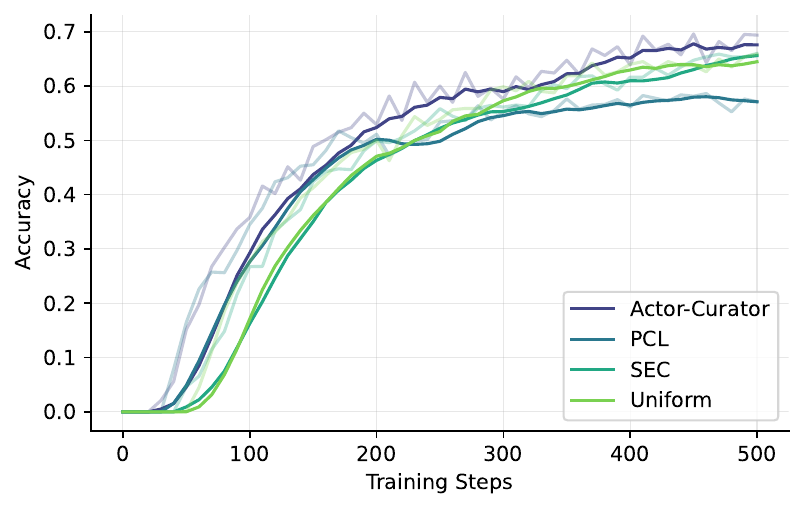}
    \caption{\textbf{ARC-1D}}
    \label{fig:llama-arc-1d}
  \end{subfigure}\hfill
  \begin{subfigure}[t]{0.40\textwidth}
    \centering
    \includegraphics[width=\linewidth]{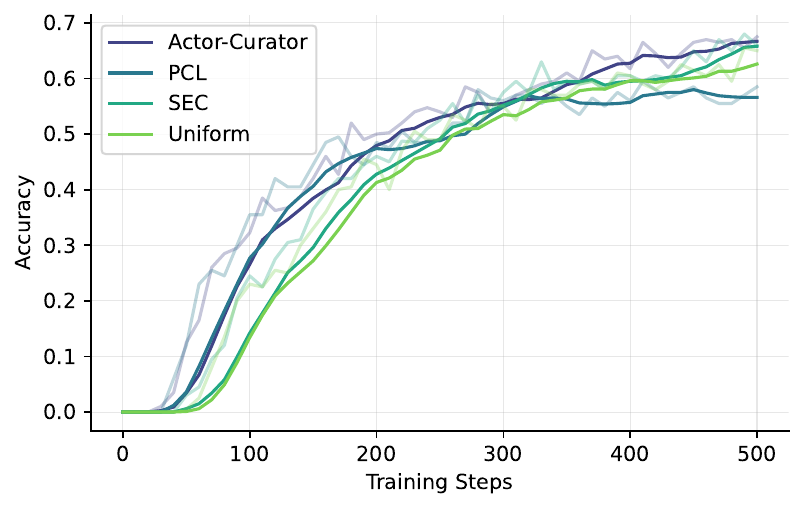}
    \caption{\textbf{ARC-hard}}
    \label{fig:llama-arc-hard}
  \end{subfigure}

  \vspace{0.7em}

  \begin{subfigure}[t]{0.40\textwidth}
    \centering
    \includegraphics[width=\linewidth]{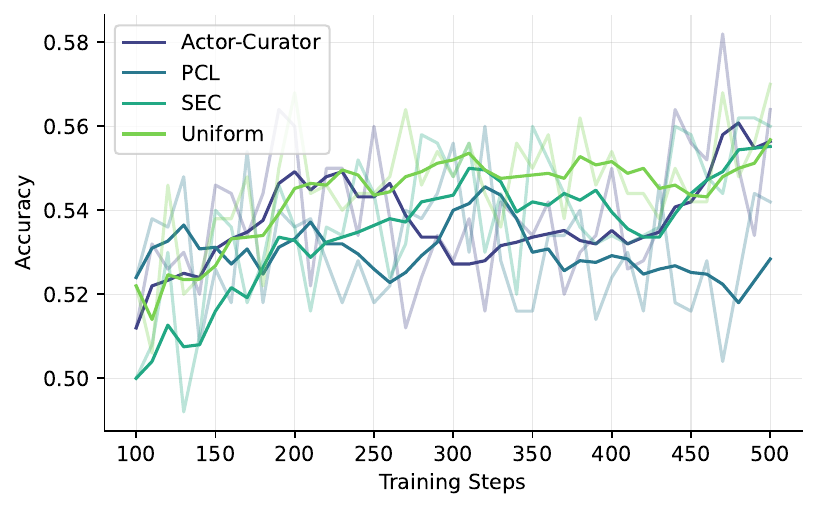}
    \caption{\textbf{MATH500}}
    \label{fig:llama-math500}
  \end{subfigure}\hfill
  \begin{subfigure}[t]{0.40\textwidth}
    \centering
    \includegraphics[width=\linewidth]{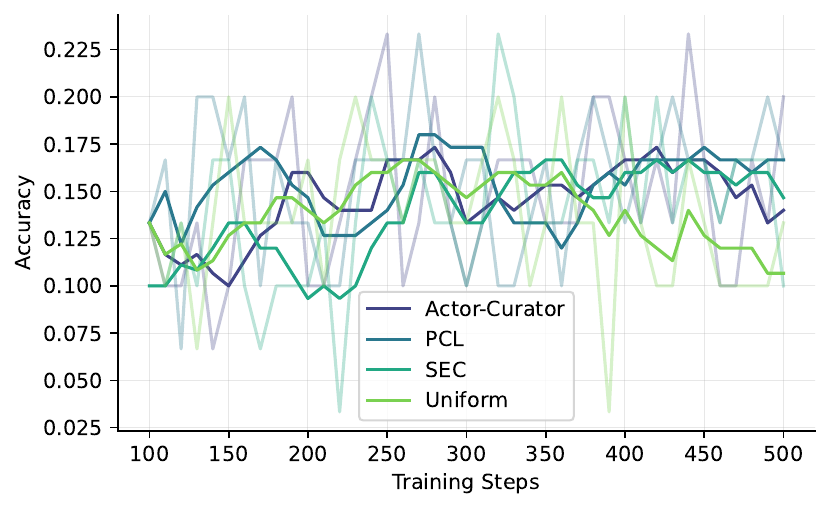}
    \caption{\textbf{AIME24}}
    \label{fig:llama-aime24}
  \end{subfigure}

\end{figure*}

\begin{figure*}[t]
  \centering
  \caption{\textbf{Training efficiency:} Actor-Curator attains significant efficiency increase with relatively low overhead on Countdown, Zebra, and ARC. 
  }
  \label{fig:efficiency-all}
  \vspace{0.5em}

  \begin{subfigure}[t]{0.50\textwidth}
    \centering
    \includegraphics[width=\linewidth]{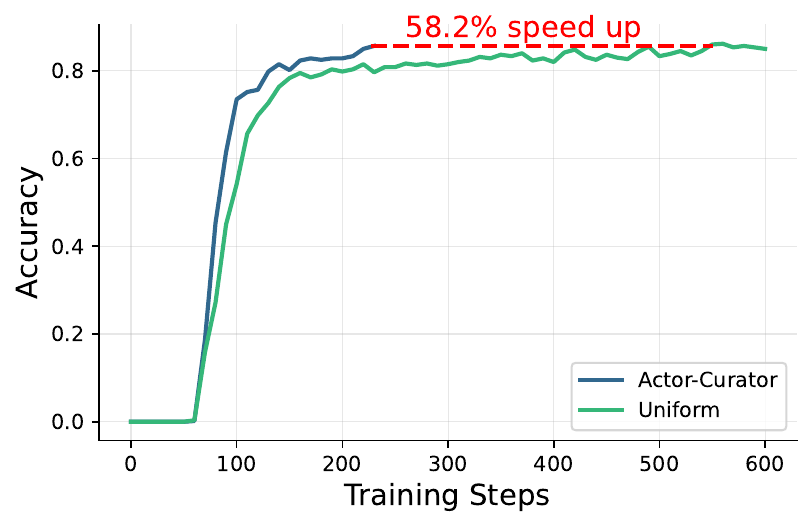}
    \caption{\textbf{Countdown}}
    \label{fig:llama-countdown}
  \end{subfigure}

  \vspace{0.7em}
  
  \begin{subfigure}[t]{0.50\textwidth}
    \centering
    \includegraphics[width=\linewidth]{sections/figs/zebra_efficiency_gain.pdf}
    \caption{\textbf{Zebra}}
    \label{fig:llama-countdown-hard}
  \end{subfigure}

  \vspace{0.7em}

  \begin{subfigure}[t]{0.50\textwidth}
    \centering
    \includegraphics[width=\linewidth]{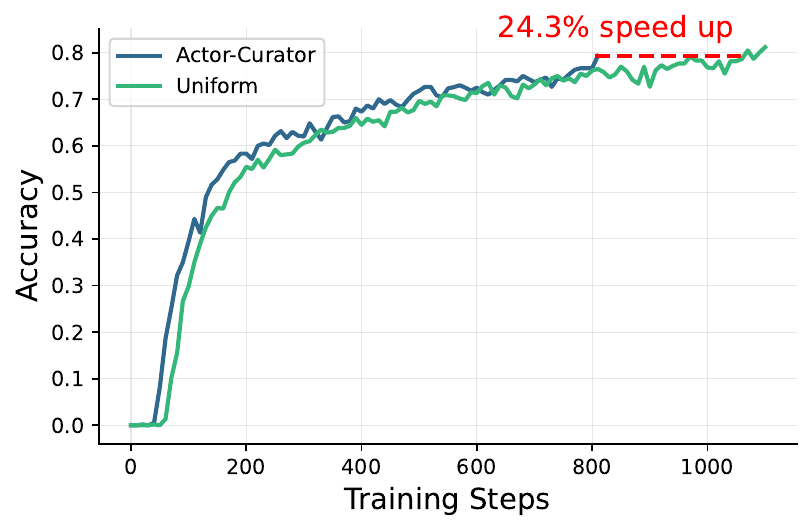}
    \caption{\textbf{ARC-1D}}
    \label{fig:llama-zebra}
  \end{subfigure}\hfill

\end{figure*}

\begin{table*}[h]
  \caption{Peak validation performance on problems within 100 training steps across models and methods. \method outperforms both other learning based methods (PCL) and methods that rely on human heuristics (SEC). \textbf{Models:} Qwen2.5 refers to Qwen2.5-3B-Base; Llama3.2 refers to Llama3.2-3B-it.}
  \label{tab:performance_full}
  \vskip 0.15in
  \begin{center}
    \begin{small}
      \begin{sc}
        \setlength{\tabcolsep}{4pt}
        \renewcommand{\arraystretch}{1.05}
        \begin{tabular}{llcccccccc}
          \toprule
          \textbf{Benchmark} & \textbf{Model}
          & \multicolumn{5}{c}{\textbf{Method}}
          & \multicolumn{2}{c}{\textbf{Improvement}} \\
          \cmidrule(lr){3-7}\cmidrule(lr){8-9}
          &
          & $\boldsymbol{\pi_{\mathrm{ref}}}$
          & \textbf{Uniform}
          & \textbf{SEC}
          & \textbf{PCL}
          & \textbf{\meth (Ours)}
          & \textbf{$+\Delta$}
          & \textbf{$+\Delta\%$} \\
          \midrule

            \multirow{2}{*}{Countdown}
            & Qwen2.5
            & 0.00 & 44.74 & 58.87 & 57.24 & 62.12 & +3.25 & +5.52 \\
            & Llama3.2
            & 0.00 & 63.12 & 62.78 & 59.62 & 66.25 & +3.13 & +4.96 \\
            
            \multirow{2}{*}{CD-hard}
            & Qwen2.5
            & 0.00 & 41.00 & 51.50 & 48.00 & 58.00 & +6.50 & +12.62 \\
            & Llama3.2
            & 0.00 & 58.50 & 58.00 & 53.00 & 60.50 & +2.00 & +3.42 \\
            \midrule
            
            \multirow{2}{*}{Zebra}
            & Qwen2.5
            & 0.00 & 35.12 & 36.00 & 34.12 & 37.62 & +1.62 & +4.50 \\
            & Llama3.2
            & 0.00 & 44.50 & 46.50 & 48.25 & 47.12 & -1.13 & -2.34 \\
            
            \multirow{2}{*}{Zebra-hard}
            & Qwen2.5
            & 0.00 & 30.50 & 27.50 & 26.00 & 34.50 & +4.00 & +13.11 \\
            & Llama3.2
            & 0.00 & 37.50 & 38.00 & 39.50 & 40.50 & +1.00 & +2.53 \\
            \midrule
            
            \multirow{2}{*}{ARC-1D}
            & Qwen2.5
            & 0.00 & 26.74 & 27.87 & 26.37 & 36.37 & +8.50 & +30.51 \\
            & Llama3.2
            & 0.00 & 27.62 & 26.75 & 34.50 & 35.25 & +0.75 & +2.17 \\
            
            \multirow{2}{*}{ARC-hard}
            & Qwen2.5
            & 0.00 & 19.50 & 18.50 & 18.50 & 31.00 & +11.50 & +58.97 \\
            & Llama3.2
            & 0.00 & 23.00 & 24.50 & 24.00 & 31.50 & +7.00 & +28.57 \\
            \midrule
            
            \multirow{2}{*}{MATH500}
            & Qwen2.5
            & 61.80 & 83.00 & 81.00 & 79.79 & 81.00 & -2.00 & -2.41 \\
            & Llama3.2
            & 41.00 & 52.20 & 52.00 & 52.40 & 53.60 & +1.20 & +2.29 \\
            
            \multirow{2}{*}{AIME24}
            & Qwen2.5
            & 3.33 & 23.33 & 20.00 & 23.33 & 30.00 & +6.67 & +28.57 \\
            & Llama3.2
            & 0.00 & 13.33 & 13.33 & 13.33 & 16.67 & +3.34 & +25.06 \\
            
          \bottomrule
        \end{tabular}
      \end{sc}
    \end{small}
  \end{center}
  \vskip -0.1in
\end{table*}

\end{document}